\documentclass[12pt]{article}

\usepackage{amsmath}
\usepackage{graphicx}
\usepackage{enumerate}
\usepackage{natbib}
\usepackage{url} 

\usepackage{amssymb} 
\usepackage{threeparttable}
\usepackage{cancel}

\usepackage{bm}
\usepackage{changes}

\usepackage[ruled]{algorithm2e}
\usepackage{graphicx}
\usepackage{subcaption}
\usepackage{multirow}
\usepackage{ntheorem}

\newtheorem{remark}{Remark}
\newtheorem{theorem}{Theorem}
\newtheorem{lemma}{Lemma}
\newtheorem*{proof}{Proof}

\usepackage{authblk}



\begin{document}

\def\spacingset#1{\renewcommand{\baselinestretch}%
	{#1}\small\normalsize} \spacingset{1.19}


\title{\bf \Large{Tight Generalization Error Bounds for Stochastic Gradient Descent in Non-convex Learning}}
	
\date{}

\small
{	\author{Wenjun Xiong$^1$, Juan Ding$^2$, Xinlei Zuo$^3$, Qizhai Li$^4$\thanks{Corresponding author. Email: liqz@amss.ac.cn}\\
		$^1$School of Mathematics and Statistics, Guangxi Normal University,\\
		$^2$School of Mathematics, Hohai University,\\
		$^3$College of Marine Life Sciences, Ocean University of China,\\
		$^4$State Key Laboratory of Mathematical Sciences, \\
		Academy of Mathematics and Systems Science, Chinese Academy of Sciences, \\
		and University of Chinese Academy of
		Sciences, China		
	}
	\maketitle
} 

\vspace{-0.3in}

\begin{abstract}
Stochastic Gradient Descent (SGD) is fundamental for training deep neural networks, especially in non-convex settings. Understanding SGD's generalization properties is crucial for ensuring robust model performance on unseen data. In this paper, we analyze the generalization error bounds of SGD for non-convex learning by introducing the Type II perturbed SGD ({\it T2pm-SGD}), which accommodates both sub-Gaussian and bounded loss functions. The generalization error bound is decomposed into two components: the trajectory term and the flatness term. Our analysis improves the trajectory term to \(O(n^{-1})\), significantly enhancing the previous \(O((nb)^{-1/2})\) bound for bounded losses, where \(n\) is the number of training samples and \(b\) is the batch size. By selecting an optimal variance for the perturbation noise, the overall bound is further refined to \(O(n^{-2/3})\). For sub-Gaussian loss functions, a tighter trajectory term is also achieved. In both cases, the flatness term remains stable across iterations and is smaller than those reported in previous
literature, which increase with iterations. This stability, ensured by {\it T2pm-SGD}, leads to tighter generalization error bounds for both loss function types. Our theoretical results are validated through extensive experiments on benchmark datasets, including MNIST and CIFAR-10, demonstrating the effectiveness of {\it T2pm-SGD} in establishing tighter generalization bounds.
\end{abstract}

\begin{center}
\begin{minipage}{0.87\textwidth}
	\textbf{Keywords:}
 Generalization error, Stochastic gradient descent, Kullback-Leibler divergence, Non-convex loss
\end{minipage}
\end{center}

	\vfill	

	\section{Introduction}
	Consider a probability space $(\Omega, {\mathfrak X}, \mu)$, and
	let $\bm X_1,\ldots,\bm X_n$ be $n$ independent observations drawn from $\mu$.
	Based on the dataset $S=\{\bm X_1,\ldots,\bm X_n\}$,
	we aim to train a deep neural network (DNN) with parameters $\bm \theta\in\Theta\subset{\cal R}^d$.
	The loss function is defined as $f: \Omega\times \Theta\rightarrow {\cal R}$.
	The parameter $\bm \theta$ is estimated by solving the following minimization problem,
	\begin{eqnarray}\label{eq1}
		\begin{split}
			\min\limits_{\bm \theta\in\Theta} \dfrac1n\sum\limits_{i=1}^n  f( \bm X_i,\bm \theta).
		\end{split}
	\end{eqnarray} 
	The generalization error, as defined in previous works \citep{mukherjee2006, bousquet2002, shalev2010, zhang2016, neyshabur2017}, is given by
	\begin{eqnarray}\label{eq2}
		\begin{split}
			ge=\mathbb E_S\left (\dfrac1n\sum\limits_{i=1}^n f(\bm X_i,\bm \theta_S)-\mathbb E_{\bm X} f(\bm X, \bm \theta_S)\right ),
		\end{split}
	\end{eqnarray}
	where $\mathbb E_{S} $ denotes the expectation with respect to the random variables in $S$, $\bm X$ is a random variable independent of $S$ with $\bm X \sim \mu$, and $\bm \theta_S$ is the solution of (\ref{eq1}) that depends on $S$.
	
	Optimizing (\ref{eq1}) is crucial for training a DNN. Stochastic gradient descent (SGD) and its variants are widely used due to their scalability with data size and strong performance in large-scale systems \citep{bottou2010, mandt2017,kleinberg2018, kawaguchi2020}.
	Among these methods, the mini-batch SGD is particularly popular, where the gradient is estimated using a randomly selected subset of $S$, called a mini-batch, during each update step. 
	The generalization error $ge$, as defined in (\ref{eq2}), quantifies both the stability and overfitting of an algorithm \citep{lopez2018, issa2019}. 
	It is frequently used to evaluate the generalization capability of a training algorithm and serves as a key metric for the effectiveness.
	
	\subsection{Literature Review}
	The generalization error of SGD has been extensively studied in the literature. The early work of \citet{Hardt2016} established the generalization error bounds for SGD using algorithmic stability, a concept previously explored by \citet{bousquet2002} and \citet{Elisseeff2005}. Since then, there have been significant progresses in advancing the understanding of SGD's stability.

	For the convex loss functions, \citet{Bassily2020} analyzed the uniform stability of SGD and provided sharper bounds, thus strengthening the results of \citet{Hardt2016}. Alternatively, by combining PAC-Bayes theory with algorithmic stability, \citet{London2016} derived the generalization error bounds for SGD and extended the results to randomized learning algorithms. \citet{meng2017} investigated the generalization error bounds of various optimization algorithms, including SGD. The characteristics of loss functions, such as convexity, smoothness, and non-convexity, introduce distinct challenges in the analysis of generalization error.

	For the non-convex smooth loss functions, \citet{kuzborskij2018} derived the data-dependent stability bounds for SGD, while \citet{zhouy2021} introduced on-average stability results and provided high-probability guarantees for the generalization bound. Additionally, \citet{roberts2021} applied the central limit theorem to derive the generalization error bounds, which require sufficiently large batch sizes. \citet{zhou2022} investigated the stability of a truncated SGD under the assumption that the loss function is bounded and Lipschitz. Furthermore, clipped SGD has gained attention as a simple yet effective method for stabilizing the training process by controlling the gradient update \citep{Gorbunov2024}.

	Recently, \citet{neu2021} introduced information-theoretic bounds for SGD in the non-convex learning by developing a perturbed SGD framework. Their key insight was to add random noise to the iterates, resulting in tighter bounds. This approach enables the application of \citet{Pensia2018}'s technique to these virtual iterates, aiding in the derivation of the bounds. The framework used surrogate analysis, perturbing the iterates with a Gaussian random variable to obtain the generalization error bound. As noted by \citet{Haghifam2023}, such ``surrogate" algorithms frequently appear in the generalization literature \citep{Hellstrom2021, Dziugaite20, Neyshabur18, Foret20, Pour22}, with ``Gaussian surrogate" being the most commonly used.
	
	Later, \citet{wang2021} refined the bound proposed by \citet{neu2021} by removing the local gradient sensitivity, resulting in a bound comprising trajectory and flatness terms. However, these two terms exhibit a complex relationship with hyperparameters such as learning rate and batch size. In this paper, we  propose a new ``surrogate" SGD that offers deeper insights into the stability and convergence properties of conventional SGD. We first analyze sub-Gaussian loss functions to establish a foundational understanding, then extend our methods to bounded loss functions, offering a comprehensive framework for deriving the generalization error bounds for SGD in the non-convex tasks.

	\subsection{Contributions and organization of the paper}

	Our contributions are fourfold. First, we introduce \textit{T2pm-SGD}, which is essential for achieving the tighter generalization error bounds. Second, we reduce the flatness term from \( \xi_{k}(\bm{\theta}_k,\sum_{t=1}^k \bm{\epsilon}_t) \) to \( \xi_{k}(\bm{\theta}_k, \bm{\epsilon}_k) \). This new flatness term not only tightens the bounds but also remains stable throughout the iterative process, in contrast to previous terms that increase rapidly. Third, we refine the trajectory term, leading to an overall tighter bound. Specifically, for bounded loss functions, we improve the trajectory term from \(O(1/\sqrt{nb})\) to \(O(1/n)\), where \( b \) denotes the batch size, which is typically small in practice. Under sub-Gaussian condition, we also achieve a tighter trajectory term. Finally, we provide explicit guidelines for selecting the optimal noise variance, a critical step in further refining and tightening the bound.

    The remainder of the paper is organized as follows. Section 2 introduces the proposed perturbed SGD approach and establishes a tight generalization error bound for non-convex learning scenarios, including sub-Gaussian and bounded loss functions. We theoretically demonstrate that our method achieves tighter bounds than existing work. In Section 3, we relax our assumptions on the gradient moments and consider clipped SGD, deriving its generalization error bounds. Section 4 presents numerical experiments to evaluate the performance of our method. Section 5 provides conclusions and discussions. Finally, Section 6 provides the technical details of the theoretical results.

	\section{Main Results}
	\subsection{Updating Rules}
	Let $\bm{\theta}_{S,1},\ldots,\bm{\theta}_{S,K}$ denote the sequence of parameter estimates generated by mini-batch SGD, where $\bm \theta_{S,\cdot}$ indicates that the solution is computed based on the data $S$.
	For $k=1,\ldots,K,$ the update rule for mini-batch SGD \citep{Shamir2013} at the $k^{th}$ iteration is given by
	\begin{eqnarray}
		\begin{split}
			\bm \theta_{S,k}=\bm \theta_{S,k-1}-\eta_k \bm g_{S,k}\big(\bm \theta_{S,k-1}\big),
		\end{split}
	\end{eqnarray}
	where $\bm g_{S,k}\big(\bm \theta_{S,k-1}\big)=\dfrac{1}{b}\sum\limits_{i\in M_k}\nabla f\big(\bm X_i,\bm \theta_{S,k-1}\big)$ and $M_k\subset \{1,\ldots,n\}$ denotes the $k^{th}$ min-batch of size $b$. Here, $\bm \theta_{S,0}$ represents the initial value of $\bm \theta_{S}$, and $\eta_k$ is the learning rate at the $k^{th}$ iteration.

	Instead of directly analyzing the generalization error of mini-batch SGD, \citet{neu2021} introduced a modified version of SGD by adding random noise to the iterates for analytical purposes. This modification allows for an information-theoretic approach to estimating the generalization error. In their modified mini-batch SGD, the pseudo-update rule at iteration $k$ is defined as
	\begin{eqnarray}\label{eq_T1}
		\begin{split}
			\tilde{\bm \theta}_{S,k}=\tilde{\bm \theta}_{S,k-1}-\eta_k \bm g_{S,k}\big(\bm \theta_{S,k-1}\big)+\bm \epsilon_k, \text{~with~} \bm \epsilon_{k}\sim N(0,\sigma_k^2\bm I_d),
		\end{split}
	\end{eqnarray}
	where $\sigma_k^2>0$ and $\bm I_d$ denotes the $d\times d$ identity matrix.
	We refer to the SGD based on this rule as Type I perturbed mini-batch SGD ({\it T1pm-SGD}).
	Notably, this update rule can also be expressed as
	$\tilde{\bm \theta}_{S,k}={\bm \theta}_{S,k}+ \sum_{t=1}^k\bm \epsilon_t$.
	As the number of iterations increases, the accumulated error from the noise term $\bm \epsilon_t$ may negatively impact the accurate assessment of SGD's generalization ability.
	To address this issue and obtain a tighter generalization error bound, we propose the following pseudo-update rule at the $k^{th}$ iteration,
	\begin{eqnarray}\label{eq_T2}
		\begin{split}
			\check{\bm \theta}_{S,k}=\bm \theta_{S,k-1}-\eta_k \bm g_{S,k}\big(\bm \theta_{S,k-1}\big)+\bm \epsilon_k, \text{~with~} \bm \epsilon_{k}\sim N(0,\sigma_k^2\bm I_d).
		\end{split}
	\end{eqnarray}
	We refer to the SGD based on this update as Type II perturbed mini-batch SGD ({\it T2pm-SGD}). It can be shown that $\check{\bm \theta}_{S,k}=\bm \theta_{S,k}+\bm \epsilon_k,$ which introduces a smaller perturbation compared to the accumulated error in {\it T1pm-SGD}, potentially leading to better preservation of generalization performance.

	\begin{remark}
		The key difference between {\it T1pm-SGD} and {\it T2pm-SGD} lies in how the noise term is introduced. In {\it T2pm-SGD}, the pseudo-parameter $\check{\bm \theta}_{S,k}$ is generated by directly adding random noise to ${\bm \theta}_{S,k}$, while in {\it T1pm-SGD}, the noise accumulates over iterations as $ \sum_{t=1}^k\bm \epsilon_t$.
		Similar to {\it T1pm-SGD}, the iterations of {\it T2pm-SGD}, represented by
		$\check{\bm \theta}_{S,k}$, are used solely for theoretical analysis and not for practical parameter updates.
		As shown in (\ref{eq_T1}), in \textit{T1pm-SGD} we have $\tilde{\bm \theta}_{S,k}={\bm \theta}_{S,k}+ \sum_{t=1}^k\bm \epsilon_t$, leading to noise accumulation over iterations. In contrast, in {\it T2pm-SGD}, as defined by (\ref{eq_T2}), we have $\check{\bm \theta}_{S,k}=\bm \theta_{S,k}+\bm \epsilon_k$, which prevents noise accumulation. As a result, the difference between $\check{\bm \theta}_{S,k}$ and $\bm \theta_{S,k}$ is much smaller than the difference between $\tilde{\bm \theta}_{S,k}$ and $\bm \theta_{S,k}$. This indicates that {\it T2pm-SGD} can more accurately approximate the true iterative process than {\it T1pm-SGD}.  
	\end{remark}

	\subsection{$ge$ of {\em {\it T2pm-SGD}}}
	
	The main challenge in establishing the upper bound of the generalization error, as outlined in (\ref{eq2}), lies in quantifying the difference between the empirical and expected losses. To address this, we apply the random subset method, which reformulates the generalization error and is particularly effective in tightening the bound. By adjusting the subset $J$, this method enables further optimization of the bound. Therefore, we employ this approach to analyze the generalization error and establish its upper bound.
	
	Let $J$ represent a subset in $\{1, \ldots, n\}$ with size $|J|$. The dataset $S$ can be divided into two subsets, $S_J$ and $S_{J^c}$, where $S_J=\{\bm X_i, i\in J\}$
	and $S_{J^c}=\{\bm X_i, i\notin J\}$. For example, if $J=\{2,3\}$, then $S_J=\{\bm X_2, \bm X_3\}$ and $S_{J^c}=\{\bm X_1, \bm X_4,\ldots, \bm X_n\}$.
	Let $\bm X_i'$ be an independent and identically distributed (i.i.d.) copy of $\bm X_i$ for $i\in J$. Define $S'_{J}=\{\bm X_j', j\in J\}$ and $S'=S_{J^c}\cup S_{J}'$. Given $S_{J^c}$, $S'_{J}$ is independent of both $S$ and $\bm \theta_{S,k}$.
	This independence enables us to express $ge$ at the $k^{th}$ iteration step as
	\begin{eqnarray*}
		\begin{split}	
			ge_k
			=\mathbb E_{S_{J^c}}\bigg\{\mathbb E_{S_J , S_J'|S_{J^c}}\Big[\dfrac{1}{|J|} \sum_{i\in J}\big(f(\bm X_i',\bm \theta_{S,k})-f(\bm X_i, \bm \theta_{S,k})\big)\Big]\bigg\},
		\end{split}
	\end{eqnarray*}	
	where $\mathbb E_{S_{J^c}} $ denotes the expectation with respect to $S_{J^c}$, and $\mathbb E_{S_J, S_J'|S_{J^c}}$ represents the conditional expectation with respect to $S_J$ and $S_J'$ given $S_{J^c}$. For any $i\in J,$ it can be shown that $\bm X_i$ is independent of $\theta_{S',k}$, and $\bm X_i'$ is independent of $\bm \theta_{S,k}$ given $S_{J^c}$. Thus, the loss functions $f(\bm X_i, \bm \theta_{S',k})$ and $f(\bm X_i',\bm \theta_{S,k})$ follow the same distribution. Hence, we obtain
	\begin{eqnarray}\label{ge_use}
		\begin{split}	
			ge_k
			=\mathbb E_{S_{J^c}}\bigg\{\mathbb E_{S_J , S_J'|S_{J^c}}\Big[\dfrac{1}{|J|} \sum_{i\in J}\big(f(\bm X_i,\bm \theta_{S',k})-f(\bm X_i, \bm \theta_{S,k})\big)\Big]\bigg\}.
		\end{split}
	\end{eqnarray}
	
	For simplicity, we denote $\bm \theta_{S,k}$ and $\bm \theta_{S',k}$ as $\bm \theta_k$ and $\bm \theta'_k$, respectively, in the following. Similarly, $\check{\bm \theta}_k$ and $\check{\bm \theta}'_k$ represent $\check{\bm \theta}_{S,k}$ and $\check{\bm \theta}_{S',k}$, respectively. Thus, the term inside the summation in (\ref{ge_use}) becomes $f(\bm X_i, \bm \theta_k')-f(\bm X_i, \bm \theta_k)$, which can be rewritten as
	\begin{eqnarray*}
		\begin{split}
			f(\bm X_i, \bm \theta_k')-f(\bm X_i, \bm \theta_k)
			&=Y_{k,i}+f(\bm X_i, \check{\bm \theta}_k')-f(\bm X_i, \check{\bm \theta}_k), ~i\in J,
		\end{split}
	\end{eqnarray*}
	where
	$Y_{k,i}=f(\bm X_i, \bm \theta_k')-f(\bm X_i, \check{\bm \theta}_k')+ f(\bm X_i, \check{\bm \theta}_k)-f(\bm X_i, \bm \theta_k)$ and $\check{\bm \theta}_k=\bm \theta_k+\bm \epsilon_k$. Define
	$$\xi_{k}(\bm \theta_k,\bm \epsilon_k)=\mathbb E_{S_{J^c}}\Big[\mathbb E_{S_J, S_J'|S_{J^c}}\Big(\dfrac{1}{|J|}\displaystyle \sum_{i\in J}Y_{k,i}\Big)\Big].$$
	This term $\xi_{k}(\bm \theta_k,\bm \epsilon_k)$ is referred to as the flatness term, which measures how flat the loss landscape is at the solution point.
	Thus, we have
	\begin{eqnarray}
		\begin{split}
			|ge_k| \leq |\xi_{k}(\bm \theta_k,\bm \epsilon_k)|+|\check{ge}_k|,
		\end{split}
	\end{eqnarray}
	where
	$\check{ge}_k=\mathbb E_{S_{J^c}}\Big\{\mathbb E_{S_J, S_J'|S_{J^c}}\Big[\dfrac{1}{|J|}\displaystyle \sum_{i\in J}\big(f(\bm X_i, \check{\bm \theta}_k')-f(\bm X_i, \check{\bm \theta}_k)\big)\Big]\Big\}$.
	
	\begin{remark}\label{r3}
		Since $\bm X_1,\ldots,\bm X_n$ are {\em i.i.d.}, given $S_{J^c}$, both $S$ and $S'$ are also {\em i.i.d.}, implying that $\bm \theta_k$ and $\bm \theta_k'$ are {\em i.i.d.} too. So,
		\[\xi_{k}(\bm \theta_k,\bm \epsilon_k)=2 \mathbb E_{S, S'}\big[f(\bm X_i, \bm \theta_k+\bm \epsilon_k)-f(\bm X_i, \bm \theta_k)\big],\]
		which represents the expected change in the loss function when the parameter  $\bm \theta_k$ is perturbed by $\bm \epsilon_k$.
		This measure is analogous to the ``local value sensitivity" described by \citet{neu2021} and \citet{wang2021}, where the corresponding value is
		\[\xi_{k}\Big(\bm \theta_k,\displaystyle \sum_{i=1}^k\bm \epsilon_i\Big)=2\mathbb E_{S, S'}\bigg[f\Big(\bm X_i, \bm \theta_k+\displaystyle \sum_{t=1}^k\bm \epsilon_t\Big)-f(\bm X_i, \bm \theta_k)\bigg],\]
		with the accumulated perturbations $\sum_{j=1}^k\bm \epsilon_j$.
		In contrast to their approach, our local value sensitivity exhibits a smaller magnitude and remains stable throughout the iterative process.
	\end{remark}

	\subsection{Upper Bound of $ge$ for Sub-Gaussian Loss}
	According to \citet{neu2021}, a loss function $f(\bm X, \bm \theta)$ is $R$-sub-Gaussian for any $\bm \theta \in \Theta$ if there exists a positive constant $R$ such that, for any $t$, 
	$$\mathbb E_{\bm X}\Big( \exp\big[t\big(f(\bm X, \bm \theta) - \mathbb E_{\bm X} f(\bm X, \bm \theta)\big) \big] \Big)\leq  \exp\left(-\dfrac{R^2t^2}{2}\right).$$
	To facilitate our analysis, we define $L_{Jk}=1$ if the subset $J$ intersects with the $k^{th}$ mini-batch, and $L_{Jk}=0$ otherwise. Specifically, for $J=\{j\}$, we have $P(L_{Jk}=1)=b/n$, where $b$ is the batch size and $n$ is the total number of data points. Let $\bm L_J=(L_{J1},\ldots,L_{Jk})$. Let $\check{\upsilon}_{s|S_{J^c},\bm \epsilon_{s-1},\bm L_J}$ denote $ \mathrm{tr}\big[\mathbb V^{S_{J^c},\bm \epsilon_{s-1},\bm L_J}\big(\bm g_{S_J,s}(\check{\bm \theta}_{s-1}-\bm \epsilon_{s-1})|\check{\bm \theta}_{s-1}\big)\big]$,
	where $\mathbb V^{S_{J^c},\bm \epsilon_{s-1},\bm L_J}$ represents the conditional variance given $S_{J^c}$, $\bm \epsilon_{s-1}$ and $\bm{L}_J$; $\mathrm{tr}(A)$ denotes the trace of matrix $A$; and $\bm g_{S_J,s}(\check{\bm \theta}_{s-1}-\bm \epsilon_{s-1})=\frac{1}{b}\sum\limits_{i\in J\cap M_s}\nabla f(\bm X_i, \check{\bm \theta}_{s-1}-\bm \epsilon_{s-1})$. Under these definitions, we establish the following conditions and results.
	\\

	\noindent \textbf{Condition} 1.1. The loss function $f(\bm X, \bm \theta)$ is $R$-sub-Gaussian for any $\bm \theta \in \Theta$.
	
	\noindent \textbf{Condition} 1.2. The expectation $\mathbb E_{\bm \epsilon_{s-1},S_{J^c},\bm L_J}\big(\check{\upsilon}_{s|S_{J^c},\bm \epsilon_{s-1},\bm L_J}\big)$ is bounded.
	
	\begin{remark}
		The sub-Gaussian condition is a common assumption in the study of non-convex learning algorithms, as it controls the tail behavior of the loss function. In our analysis, we require that \( \mathbb E_{\bm \epsilon_{s-1},S_{J^c},\bm L_J}\big(\check{\upsilon}_{s|S_{J^c},\bm \epsilon_{s-1},\bm L_J}\big) \) is bounded. By utilizing the inequality $\mathbb EX^2\geq (\mathbb EX)^2$, we find that this condition is weaker than the one imposed by \citet{wang2021}, who require the boundedness of \( \mathrm{tr}\left[ \mathbb{V}\left( \bm g_{S_J, s}(\bm{\theta}_{s-1}) \right) \right] \). Therefore, our method imposes less restrictive conditions on the loss function and its gradient, potentially leading to tighter generalization bounds.
	\end{remark}

	\begin{theorem} \label{th-sub-1.2} Under Conditions 1.1 and 1.2, the absolute value of the generalization error $|ge_k|$ is bounded by
		\begin{eqnarray}\label{eq-th4}
			|ge_k|\leq
			\min_{J}\Bigg\{|\xi_{k}(\bm \theta_k,\bm \epsilon_k)|+\mathbb E_{S_{J^c},\bm L_J}\sqrt{\dfrac{R^2d}{|J|} \displaystyle \sum_{s=1}^{k}\mathbb E_{\bm \epsilon_{s-1}|S_{J^c},\bm L_J}\log\Big(1+\eta^2_s\dfrac{\check{\upsilon}_{s|S_{J^c},\bm \epsilon_{s-1},\bm L_J}}{d \sigma^2_s}\Big)}\Bigg\}.
		\end{eqnarray}		
		If $J=\{j\}$, then we have 
		\begin{eqnarray}\label{eq-th4.1}
			\begin{split}
				|ge_k|\leq|&\xi_{k}(\bm \theta_k,\bm \epsilon_k)|\\
				&+\mathbb E_{S_{\{j\}^c}, \bm L_{\{j\}}} \sqrt{R^2d \displaystyle \sum_{s=1}^{k}\mathbb E_{\bm \epsilon_{s-1}|S_{\{j\}^c},\bm L_{\{j\}}}\Big[L_{\{j\}s}\log\Big(1+\dfrac{\eta^2_s}{d \sigma^2_s}\check{\upsilon}_{s|S_{\{j\}^c},\bm \epsilon_{s-1},\bm L_{\{j\}}}\Big)\Big]}.
			\end{split}
		\end{eqnarray}
	\end{theorem}

	\begin{remark}
		Compared to existing results, the {\it T2pm-SGD} method significantly reduces the flatness term $|\xi_{k, J}(\bm \theta_k,\bm \epsilon_k)|$, which is much smaller than $|\xi_{k}(\bm \theta_k, \sum_{i=1}^k\bm \epsilon_i)|$ as reported by \citet{wang2021}.      In (\ref{eq-th4}), the second term is called the trajectory term, which also shows a clear improvement. By setting $J =\{1,\ldots,n\}$ and applying the inequality $\mathbb E(\sqrt{X})\leq \sqrt{\mathbb E(X)}$ for any non-negative random variable $X$, along with $L_{Js}=1$ for $s=1,\ldots,k$, the second part of the bound in (\ref{eq-th4}) becomes smaller than
		\begin{eqnarray}\label{eq-th4.11}
			\mathbb E\sqrt{\dfrac{R^2d}{n} \displaystyle \sum_{s=1}^{k}\mathbb E_{\bm \epsilon_{s-1}}\log\Big(1+\dfrac{\eta^2_s}{d \sigma^2_s}\check{\upsilon}_{s|\bm \epsilon_{s-1}}\Big)}.
		\end{eqnarray}
		By substituting $\bm \theta_{s-1}=\check{\bm \theta}_{s-1}-\bm \epsilon_{s-1}$ and using the inequality $\mathbb E(X^2)\geq [\mathbb E(X)]^2$, the trajectory term (\ref{eq-th4.11}) is further bounded by $\sqrt{\dfrac{R^2d}{n} \displaystyle \sum_{s=1}^{k}\mathbb E\log\Big(1+\dfrac{\eta^2_s}{d \sigma^2_s} \mathrm{tr}\big[\mathbb V\big(\bm g_{S_J,s}(\bm \theta_{s-1})\big)\big] \Big)},$ as reported by \citet{wang2021}.  Therefore, both terms in our bound are smaller than the corresponding terms in \citet{wang2021}, highlighting the advantages of {\it T2pm-SGD} under same assumption.
	\end{remark}
	
	\begin{remark}
		By setting $J =\{j\}$ and applying the inequality $\mathbb E(\sqrt{X})\leq \sqrt{\mathbb E(X)}$ for any non-negative random variable $X$, along with $P(L_{Js}=1)=b/n$ for $s=1,\ldots,k$, the second part of the bound in (\ref{eq-th4.1}) becomes smaller than 
		\begin{eqnarray}\label{eq-th4.1134}
			\mathbb E_{S_{\{j\}^c}} \sqrt{\dfrac{bR^2d}{n} \displaystyle \sum_{s=1}^{k}\mathbb E_{\bm \epsilon_{s-1}|S_{\{j\}^c},\bm L_{\{j\}}}\Big[\log\Big(1+\dfrac{\eta^2_s}{d \sigma^2_s}\check{\upsilon}_{s|S_{\{j\}^c},\bm \epsilon_{s-1},\bm L_{\{j\}}}\Big)\big|L_{\{j\}s}=1\Big]},
		\end{eqnarray}
		By applying the inequality $\ln(1+x)<x$ for $x>0$, the trajectory term (\ref{eq-th4.1134}) is further bounded by $$\mathbb E_{S_{\{j\}^c}} \sqrt{\dfrac{R^2}{nb} \displaystyle \sum_{s=1}^{k}\frac{\eta^2_s}{\sigma^2_s}\mathbb E_{\bm \epsilon_{s-1}|S_{\{j\}^c},\bm L_{\{j\}}}\Big(\mathrm{tr}\big[\mathbb V^{S_{J^c},\bm \epsilon_{s-1},\bm L_J}\big( \nabla f(\bm X_j, \check{\bm \theta}_{s-1}-\bm \epsilon_{s-1}) |\check{\bm \theta}_{s-1}\big)\big]\big|L_{\{j\}s}=1\Big)},$$ which is of the order $O\big(\dfrac{1}{\sqrt{nb}}\big).$
	\end{remark}

	Assume that $\mathbb E\big(\triangle f(\bm \theta_k)\big)=O(1)$, $\sigma_k=O(n^{-\gamma})$. Then the bound in $(\ref{eq-th4.11})$ is of order $O(n^{-2\gamma})+O(b^{-1/2}n^{-1/2+\gamma})$, and can be further reduced to $O\big((nb)^{-1/3}\big)$ by appropriately selecting the parameter $\gamma$.

	For a general noise characterized by $\bm \epsilon_k\sim N(0,\bm \Sigma_k)$, we derive a bound analogous to (\ref{eq-th4}), as summarized in the following theorem.
	\begin{theorem} \label{th-sub-3.2}
		Under Conditions 1.1 and 1.2, if $\bm \epsilon_s\sim N(0,\bm \Sigma_s)$ for $s=1,\ldots,k$, the absolute value of the generalization error $|ge_k|$ is bounded by
		\begin{eqnarray*}
			\begin{split}
				\min_{J}\Bigg\{|\xi_{k}(\bm \theta_k,\bm \epsilon_k)|+\mathbb E_{S_{J^c},\bm L_J}\sqrt{\dfrac{R^2d}{|J|} \displaystyle \sum_{s=1}^{k}\mathbb E_{\bm \epsilon_{s-1}|S_{J^c},\bm L_J}\log\Big(1+\dfrac{\eta^2_s}{d |\bm \Sigma_s|^{1/d}}\check{\upsilon}_{s|S_{J^c},\bm \epsilon_{s-1},\bm L_J}\Big)}\Bigg\}.
			\end{split}
		\end{eqnarray*}
	\end{theorem}
	
	\subsection{Upper Bound of $ge$ for Bounded Loss}
In this section, we focus on bounded loss functions, which inherently limit the influence of extreme values and thereby help stabilize gradient-based optimization methods such as SGD. Notable examples of bounded loss functions include the Hinge Loss \citep{jin2014, Fu2024} and the Huber Loss \citep{Huber1964, nguyen2024}. 
	Let \(P_{\check{\bm{\theta}}'_k|S, S'}\) and \(P_{\check{\bm{\theta}}_k|S, S'}\) represent the conditional density functions of \(\check{\bm{\theta}}'_k\) and \(\check{\bm{\theta}}_k\), given $S$ and $S'$, respectively. \(\check{ge}_k\) can be expressed as
	\[\check{ge}_k=\mathbb E_{S, S'}\bigg[\dfrac{1}{|J|}\displaystyle \sum_{i\in J}\int f(\bm X_i,\bm \theta) \Big(P_{\check{\bm{\theta}}'_k|S, S'}(\bm \theta)-P_{\check{\bm{\theta}}_k|S, S'}(\bm \theta)\Big)d\bm \theta\bigg].\]
	To evaluate $\check{ge}_k$, we assess the similarity between the densities $P_{\check{\bm{\theta}}'_k|S, S'}$ and $P_{\check{\bm{\theta}}_k|S, S'}$, measured using the total variation (TV) distance and the Kullback-Leibler (KL) divergence.
	The TV distance is defined as $D_{TV}(p,q)=\dfrac{1}{2}\int |p(x)-q(x)|dx$, and the KL divergence is defined as $D_{KL}(p,q)=\int p(x)\ln\big(p(x)/q(x)\big)dx$, where $p(\cdot)$ and $q(\cdot)$ are two probability density functions.

	In Theorem \ref{th-sub-1} below, we derive a tight bound for the perturbed SGD using the TV distance and KL divergence, assuming the loss function is bounded. This result marks a significant improvement over previous works in this domain. Before presenting the theorem, we introduce several key terms. 
	Let $P_{\bm h'_{k}|S, S'}$ denote the conditional density function of $\bm h'_{k}=\bm \theta'_{k-1}-\eta_{k} \bm g_{S,k}(\bm \theta'_{k-1})$ given $S$ and $S'$. Additionally, let $P_{\bm \iota_{t}'|\bm h'_{k}, S, S'}$ represent the conditional density function of $\bm \iota_{t}'=t[\bm g_{S',k}(\bm \theta'_{k-1})-\bm g_{S,k}(\bm \theta'_{k-1})]+s_{k}W'_t$ given $\bm h'_{k}, S$ and $S',$ where $W'_t\sim N(0,t\bm I_d)$ for $t\in [0,\eta_k]$. Define $H_{J,s}(\bm \iota_{t}')= \sum_{j\in J}\mathbb E\big[\big(\nabla f(\bm X_{j}', \bm \theta_{s-1}')-\mathbb E\nabla f(\bm X_{j}, \bm \theta_{s-1}')\big)|\bm \iota_{t}', \bm h'_{k}, S, S'\big]$. We impose the following conditions.\\
	
	\noindent \textbf{Condition} 2.1. The loss function $f(\bm X,\bm \theta)$ is $c_0$-bounded; that is, there exists a constant $c_0>0$ such that $\sup_{\bm X,\bm \theta} |f(\bm X,\bm \theta)|\leq c_0$.
	
	\noindent \textbf{Condition} 2.2. There exists a constant $c_1>0$ such that $\|H_{J,s}(\bm \iota_{t}')\|\leq c_1.$
	\begin{remark}\label{rem08}
		Condition 2.1 is common in the theoretical analyses; for example, the 0-1 loss function in classification is bounded within $[0,1]$. In practice, model outputs are often constrained using techniques like softmax functions or regularization, effectively keeping loss values within a reasonable range.
		The bound in equation (\ref{bound1}) depends on $H_{\{j\},s}(w)$, which involves gradients of the loss function. While many existing works require either $\|\nabla f(\bm X, \bm \theta)-\nabla f(\bm X', \bm \theta)\|\leq c_1$ or bounded second moment \citep{kuzborskij2018, zhouy2021} to ensure bounded results, our Condition 2.2 imposes a restriction only on 
		$\sum_{j\in J}\mathbb E\big[\big(\nabla f(\bm X_{j}', \bm \theta_{s-1}')-\mathbb E\nabla f(\bm X_{j}, \bm \theta_{s-1}')\big)|\bm \iota_{t}', \bm h'_{k}, S, S'\big]$. By Jensen's inequality, we find $\big\|\mathbb E\big[\nabla f(\bm X_{j}', \bm \theta_{s-1}')-\mathbb E  \nabla f(\bm X_{j}, \bm \theta_{s-1}')|\bm \iota_{t}', \bm h'_{k}, S, S'\big]\big\|^2$ is smaller than the expectations of squared norms of the gradients, $\mathbb E\big\|\big(\nabla f(\bm X_{j}', \bm \theta_{s-1}')-\mathbb E\nabla f(\bm X_{j}, \bm \theta_{s-1}')\big)|\bm \iota_{t}', \bm h'_{k}, S, S'\big\|^2$, making Condition 2.2 less restrictive than those requiring bounded gradient differences.
	\end{remark}
	
	\begin{theorem}\label{th-sub-1}  Under Conditions 2.1 and 2.2, the absolute value of the generalization error $|ge_k|$ is bounded by
		\begin{eqnarray}\label{bound1}
			|\xi_{k}(\bm \theta_k,\bm \epsilon_k)|+\dfrac{c_0c_1}{n} \sqrt{\displaystyle \sum_{s=1}^k\dfrac{2\delta_s\eta_s}{\sigma^2_{s}}},
		\end{eqnarray}
		where $\delta_s=e^{-\frac{1}{2(1-\alpha_s)}\eta_s^{1-\alpha_s}}\int_{0}^{\eta_s }e^{\frac{1}{2(1-\alpha_s)}u^{1-\alpha_s}}du$ with $\alpha_s\in (0,1)$.	
	\end{theorem}

	\begin{remark} \label{rem1.2}
		Our flatness term $\xi_{k, J}(\bm \theta_k,\bm \epsilon_k)$ shows consistent stability throughout iterations, unlike the rapidly increasing flatness term reported in \citet{wang2021}, as illustrated in Figure 1. This stability arises because {\it T2pm-SGD} incorporating smaller perturbations with $\check{\bm \theta}_{S,k}=\bm \theta_{S,k}+\bm \epsilon_k,$ whereas {\it T1pm-SGD} uses $\tilde{\bm \theta}_{S,k}={\bm \theta}_{S,k}+ \sum_{t=1}^k\bm \epsilon_t$. 
		The magnitude of $|\xi_{k, J}(\bm \theta_k,\bm \epsilon_k)|$ can be controlled by selecting appropriate noise levels. Specifically, given $$\xi_{k}(\bm \theta_k,\bm \epsilon_k)=\mathbb E[f(\bm X_i, \bm \theta_k')-f(\bm X_i, \bm \theta_k'+\bm \epsilon_k)+ f(\bm X_i, \bm \theta_k+\bm \epsilon_k)-f(\bm X_i, \bm \theta_k)],$$ applying a Taylor expansion yields $$\mathbb E[f(\bm X_i, \bm \theta_k+\bm \epsilon_k)-f(\bm X_i, \bm \theta_k)]=\dfrac{1}{2}\sigma^2_{k}\mathbb E\big(\triangle f(\bm \theta_k)\big)\big(1+o(1)\big),$$ where $\triangle$ denotes the Laplacian operator, defined as $\triangle f(x_1,\ldots,x_d)= \sum_{i=1}^d \dfrac{\partial^2 f}{\partial x_i^2},$ a second-order differential operator. If $\mathbb E\big(\triangle f(\bm \theta_k)\big)=O(1)$, then $\xi_{k, J}(\bm \theta_k,\bm \epsilon_k)=O(\sigma^2_{k}).$ 
	\end{remark}

	\begin{remark} In general, the step size $\eta_s$ plays a critical role in determining the generalization error bounds. Notably, $\delta_s$ is smaller than the step size $\eta_s$. Since $e^{\frac{1}{2(1-\alpha_s)}u^{1-\alpha_s}}\leq e^{\frac{1}{2(1-\alpha_s)}\eta_s^{1-\alpha_s}},\forall u\leq \eta_s$, it follows that $\delta_s=e^{-\frac{1}{2(1-\alpha_s)}\eta_s^{1-\alpha_s}}\int_{0}^{\eta_s }e^{\frac{1}{2(1-\alpha_s)}u^{1-\alpha_s}}du\leq \eta_s$.
	\end{remark}

	The complete proof of Theorem \ref{th-sub-1} is presented in Section 6.4. Below, we present a brief outline. Assuming the loss function is $c_0$-bounded and applying Pinsker's inequality, we derive that $\check{ge}_k\leq 2c_0D_{TV}(P_{\check{\bm{\theta}}'_k|S, S'}\|P_{\check{\bm{\theta}}_k|S, S'}) \leq c_0 \sqrt{2D_{KL}(P_{\check{\bm{\theta}}'_k|S, S'}\|P_{\check{\bm{\theta}}_k|S, S'})}.$ Next, we show that the KL divergence $D_{KL}(P_{\check{\bm{\theta}}'_k|S, S'}\|P_{\check{\bm{\theta}}_k|S, S'})$
	is bounded by $D_{KL}(P_{\check{\bm{\theta}}'_{k-1}|S, S'}\|P_{\check{\bm{\theta}}_{k-1}|S, S'})$ plus a small additional term of $\dfrac{\delta_k\eta_k}{n^2\sigma^2_{k}}c_1^2.$ This result, supported by Lemma 1 (see Section 6.1), is crucial. A key challenge is bounding the KL divergence between the densities $ P_{\bm \iota_{\eta_k}'|\bm h'_{k}, S, S'}$ and $ P_{\mathring{\bm \iota}_{\eta_k}|\bm h'_{k}, S, S'}$, where $ P_{\mathring{\bm \iota}_{\eta_k}|\bm h'_{k}, S, S'}$ represents the density of $s_kW_k$ given $\bm h'_{k}$, $S$, and $S'$. Both densities satisfy a Fokker-Planck equation, allowing us to derive the differential inequality needed for the bound. Finally, we use induction to establish the generalization error bound.

	As shown in Remark \ref{rem1.2}, we have $\mathbb E\left[f(\bm X_i, \bm \theta_k+\bm \epsilon_k)-f(\bm X_i, \bm \theta_k)\right]=\dfrac{1}{2}\sigma^2_{k} \mathbb E\big(\triangle f(\bm \theta_k)\big)\big(1+o(1)\big).$ If $\mathbb E\big(\triangle f(\bm \theta_k)\big)=O(1)$, then $\xi_{k}(\bm \theta_k,\bm \epsilon_k)=O(\sigma^2_{k})$.  Furthermore, if $\sigma_{k}=O(n^{-\gamma})$, then $\xi_{k}(\bm \theta_k,\bm \epsilon_k)=O(n^{-2\gamma}).$ These results are formally presented in Theorem \ref*{th-sub-1.1}, with the proof omitted for brevity.
	
	\begin{theorem}\label{th-sub-1.1} Under the assumptions of Conditions 2.1 and 2.2, and given that $\mathbb E\big(\triangle f(\bm{\theta}_k)\big)=O(1)$, if $\sigma_{s}=O(n^{-\gamma})$ for all $s \leq k$ with some positive constant $\gamma$, the generalization error satisfies \[ge_k=O(n^{-2\gamma})+O(n^{-(1-\gamma)}). \]
		Specifically, choosing $\gamma=1/3$ results in the fastest decay rate, yielding \[ge_k=O(n^{-2/3}).\]
		Moreover, if the variance $\sigma_{k}^2$ is iteration-dependent, such that $\sigma_{k}=O(h_{2k}n^{-1/3})$ with $h_{2k}$ depending on $k$ (for example, $h_{2k}=\sqrt{\eta_k}),$ then $\xi_{k}(\bm \theta_k,\bm \epsilon_k)=O(h^2_{2k}n^{-2/3}).$ In this scenario, the generalization error becomes
		\[ge_k=n^{-2/3}\left[O( h^2_{2k})+O\left(\sqrt{\displaystyle \sum_{s=1}^k \dfrac{\delta_s\eta_s}{h^2_{2s}}}\right)\right].\]
	\end{theorem}
	
	\begin{remark}  \citet{wang2021} derived a bound of order $O\big((nb)^{-1/3}\big)$ in their Theorem 2, assuming that $\sigma^2_{k}$ is independent of $k$. Our results improve this to $O\big(n^{-2/3}\big)$ under the $c_0$-bounded condition. 
	\end{remark}

	We consider a tight upper bound on the generalization error using isotropic Gaussian perturbations. To generalize these results for SGD with the \textit{T2pm-SGD} rule, we now consider a more general form of perturbation, where
	$\bm \epsilon_i\sim N(0,\bm \Sigma_i)$ for $i=1,\ldots,k$.
	Let $\underline{\lambda}_{i}$ and $\overline{\lambda}_{i}$ denote the smallest and largest eigenvalues of the matrix $\bm \Sigma_i$, respectively.
	We then obtain the following result.
	
	\begin{theorem}\label{th-sub-2}
		Under Conditions 2.1 and 2.2, if $\bm \epsilon_s\sim N(0,\bm \Sigma_s)$ for $s=1,\ldots,k$, the absolute value of the generalization error $|ge_k|$ is bounded by
		\begin{eqnarray}\label{bound-th3.2-0}
			|\xi_{k}(\bm \theta_k,\bm \epsilon_k)|+\dfrac{c_0c_1}{n}\sqrt{\displaystyle \sum_{s=1}^k\dfrac{2\zeta_s\eta_s}{\underline{\lambda}_{s}} },
		\end{eqnarray}
		where $\zeta_s=e^{-\frac{\underline{\lambda}_s}{2\overline{\lambda}_s(1-\alpha_s)}\eta_s ^{1-\alpha_s}}\int_{0}^{\eta_s }e^{\frac{\underline{\lambda}_s}{2\overline{\lambda}_s(1-\alpha_s)}u^{1-\alpha_s}}du$ with $ \alpha_s\in (0,1)$.
	\end{theorem}

	\begin{remark}
		The upper bound in (\ref{bound-th3.2-0}) presents two main advantages. First, the trajectory term achieves a faster convergence rate of $O(n^{-1})$, improving upon the $O\big((nb)^{-1/2}\big)$ rate reported by \citet{neu2021} under the $c_0$-bounded condition.  Second, the flatness term $|\xi_{k, J}(\bm \theta_k,\bm \epsilon_k)|$ remains stable throughout iterations and is smaller than those in \citet{neu2021} and \citet{wang2021}, as discussed in Remark \ref{r3}. 
		
	\end{remark}
	
	\begin{remark}\label{rm11}
		Note that $\mathbb E[f(\bm X_i, \bm \theta_k+\bm \epsilon_k)-f(\bm X_i, \bm \theta_k)]=\dfrac{1}{2}\mathbb E\big[\bm \epsilon_k^{\top} \big(\bigtriangledown^2 f(\bm \theta_k)\big)\bm \epsilon_k\big]\big(1+o(1)\big)$, which is $\dfrac{1}{2}\mathrm{tr}\left(\bigtriangledown^2 f(\bm \theta_k)\mathbb E(\bm \epsilon_k\bm \epsilon_k^{\top})\right)\big(1+o(1)\big).$ Since $\bm \Sigma_k=\mathbb E(\bm \epsilon_k\bm \epsilon_k^{\top})$, we obtain 
		\[\dfrac{1}{2}\mathrm{tr}\big(\bigtriangledown^2 f(\bm \theta_k)\mathbb E\bm \epsilon_k\bm \epsilon_k^{\top}\big)=\dfrac{1}{2}\mathrm{tr}\big(\bigtriangledown^2 f(\bm \theta_k)\bm \Sigma_k\big). \]
		Since $\bm \Sigma_k$ is positive definite, we have the following inequality
		\[\underline{\lambda}_{k}\big|\mathrm{tr}\big(\mathbb E\bigtriangledown^2 f(\bm \theta_k)\big)\big| \leq \big|\mathrm{tr}\big(\mathbb E\bigtriangledown^2 f(\bm \theta_k)\bm \Sigma_k\big)\big|\leq \overline{\lambda}_{k} \big|\mathrm{tr}\big(\mathbb E\bigtriangledown^2 f(\bm \theta_k)\big)\big|.\]
		Thus, if $\overline{\lambda}_{k}=O(n^{-2\gamma})$ and $\mathbb E\big(\triangle f(\bm \theta_k)\big)=O(1)$, we have $\xi_{k}(\bm \theta_k,\bm \epsilon_k)=O(n^{-2\gamma}).$
		Consequently, the generalization error can be expressed as \[ge_k=O(n^{-2\gamma})+O(n^{-(1-\gamma)}). \]
		Setting $\gamma = 1/3$ results in the fastest decay rate, leading to \[ge_k=O(n^{-2/3}).\]
		
	\end{remark}

\section{Clipped SGD and Its Generalization Error Bounds}

SGD is a cornerstone algorithm in machine learning, widely used for training deep neural networks. However, it faces significant challenges, particularly in the presence of noisy gradients, such as gradient explosion or instability caused by heavy-tailed noise, which can hinder convergence \citep{zhang2020}. Recent studies have emphasized the importance of adaptive algorithms in addressing these issues \citep{nguyen2023}. An effective technique is gradient clipping, a simple yet powerful method that stabilizes the training process by controlling the gradient update.

Gradient clipping controls the gradient update, preventing divergence due to excessively large steps, particularly when gradients are unbounded or do not satisfy Lipschitz continuity \citep{Mai2021}. In such scenarios, there is a significant demand to investigate the corresponding generalization error, as the stability of clipped SGD remains underexplored. 
The theoretical foundation for analyzing clipped SGD lies in understanding how it behaves in relation to the generalization error of machine learning models. In particular, we are interested in the upper bound for the generalization error when gradient clipping is applied.

Gradient clipping restricts gradient updates to a specified threshold, preventing large gradients from dominating the update process. Specifically, in the \( k^{\text{th}} \) iteration, the clipped gradient is defined as
\[
{\tilde{\bm g}}_{S,k}(\bm{\theta}_{S,k-1}) = \min\left(1, \frac{A}{\|\bm{g}_{S,k}\|_2}\right) \bm{g}_{S,k},
\]
where \( \bm{g}_{S,k} \) is the gradient at the \( k^{\text{th}} \) iteration, \( \|\bm{g}_{S,k}\|_2 \) denotes its \( L_2 \)-norm, and \( A \) is a predefined clipping threshold. The update rule for clipped stochastic gradient descent (SGD) at the \( k^{\text{th}} \) iteration is given by
\[
\bm{\theta}_{S,k} = \bm{\theta}_{S,k-1} - \eta_k \bm{\tilde{g}}_{S,k}(\bm{\theta}_{S,k-1}),
\]
where \( \eta_k \) denotes the learning rate at iteration \( k \). The function \( \bm{\tilde{g}}_{S,k} \) effectively caps the gradient magnitude, preventing it from becoming excessively large. For simplicity, we retain the notation from previous sections.

Previous sections typically assume that the gradient of the loss function has bounded first-order or second-order moments, which are essential for deriving tight generalization bounds, as discussed in Remark \ref{rem08}. Specifically, Conditions 1.2 and 2.2 ensure bounded second-order moments for sub-Gaussian and bounded losses, respectively. To broaden the applicability of our results, we establish generalization error bounds without relying on these restrictive conditions. We assume only that the loss function is \( R \)-sub-Gaussian (Condition 1.1) or \( c_0 \)-bounded (Condition 2.1), thereby enabling the direct derivation of generalization bounds for clipped SGD.


\begin{theorem} \label{2th-sub-1.2} 
	Under Condition 1.1, the absolute value of the generalization error \( |ge_k| \) for clipped SGD is bounded by
	\[
	|ge_k| \leq |\xi_{k}(\bm \theta_k, \bm \epsilon_k)| + \sqrt{\dfrac{R^2d}{n} \sum_{s=1}^{k} \log\left(1 + \dfrac{A^2 \eta_s^2}{\sigma_s^2}\right)}.
	\]	
	In perturbations are modeled as \( \bm \epsilon_k \sim N(0, \bm \Sigma_k) \), then we have	
	\[
	|ge_k| \leq |\xi_{k}(\bm \theta_k, \bm \epsilon_k)| + \sqrt{\dfrac{R^2d}{n} \sum_{s=1}^{k} \log\left(1 + \eta_s^2 \dfrac{A^2}{|\Sigma_s|^{1/d}}\right)}.
	\]
\end{theorem}

\begin{remark}\label{rm_clip1}
    This result, based solely on Condition 1.1 (where the loss function is \(R\)-sub-Gaussian), demonstrates that generalization error bounds can still be derived without the need for assumptions such as bounded second-order gradients. The bound further illustrates that gradient clipping effectively mitigates the impact of large gradients on the final generalization error.
\end{remark}

For loss functions that are \(c_0\)-bounded, we can similarly derive bounds for the generalization error under clipped SGD. 

\begin{theorem} \label{th-c0-bound} 
	Under Condition 2.1, the absolute value of the generalization error \( |ge_k| \) for clipped SGD is bounded by	
	\[
	|ge_k| \leq |\xi_{k}(\bm{\theta}_k, \bm{\epsilon}_k)| + 2c_0 \sqrt{\sum_{j=1}^{k} \dfrac{2\delta_j A^2 b^2 \eta_j}{n^2 \sigma_j^2}}.
	\]
    If perturbations take the form \( \bm{\epsilon}_k \sim N(0, \bm{\Sigma}_k) \), then the bound becomes
    \[
	|ge_k| \leq |\xi_{k}(\bm{\theta}_k, \bm{\epsilon}_k)| + 2c_0 \sqrt{\sum_{s=1}^{k} \dfrac{2\zeta_s A^2 b^2 \eta_s}{n^2 \underline{\lambda}_s}}.
	\]	
\end{theorem}

\begin{remark}\label{rm_clip2}
 This result shows that for bounded loss, we can still derive a generalization error upper bound without additional conditions on the gradient, such as smoothness or Lipschitz continuity.
Clipped SGD acts as regularization by capping gradient updates, reducing the impact of large gradients caused by noisy or outlier data points \citep{wu2023}. This is especially useful when gradients follow heavy-tailed distributions or other non-standard behaviors.
In deep neural networks or models prone to unstable updates, clipped SGD stabilizes optimization and ensures more controlled updates, whereas standard SGD may fail due to large gradients.   
\end{remark}

The theoretical results presented in this section are supported by the empirical results shown in Figure 2, which provide insights into the practical behavior of clipped SGD.

	\section{Experiment}
	We evaluate the properties of the proposed bounds for stochastic gradient descent (SGD) using two neural network architectures: a three-layer multilayer perceptron (MLP) on the MNIST dataset and the convolutional network AlexNet on the CIFAR-10 dataset. The detailed network parameters are provided in the Appendix.  	
	For both datasets, we utilize the cross-entropy loss function implemented in the ``torch.nn" library to train the models. The learning rates are set to 0.01 for the MLP and 0.001 for AlexNet, respectively. 
	
First, we compute the generalization error bounds derived using the {\it T2pm-SGD} method, as presented in Theorem 4, and also provide the corresponding bounds using the {\it T1pm-SGD} method for comparison. For simplicity, the constant \(R\) in both bounds is set to 1. These evaluations are conducted on the MNIST and CIFAR-10 datasets. The perturbations are generated as Gaussian noise with a standard deviation of \(\sigma_{k} = 0.005\). The results are reported in Figure \ref{fig1}.

\begin{figure}[htbp]
	\centering
	\begin{subfigure}[b]{0.48\textwidth} 
		\centering
		\includegraphics[width=\textwidth]{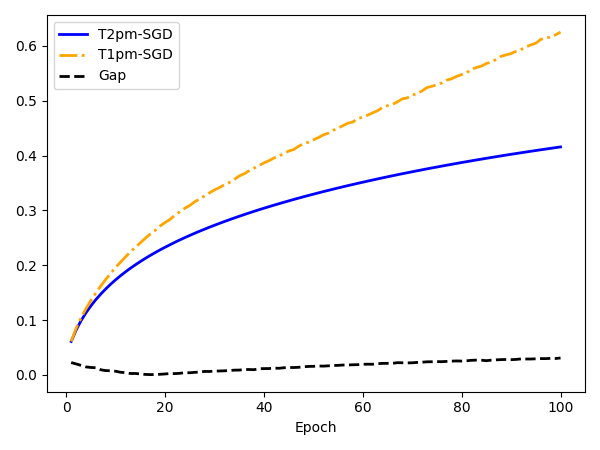}
		\caption{MNIST}
		\label{fig:sub1}
	\end{subfigure}
	\hfill
	\begin{subfigure}[b]{0.48\textwidth} 
		\centering
		\includegraphics[width=\textwidth]{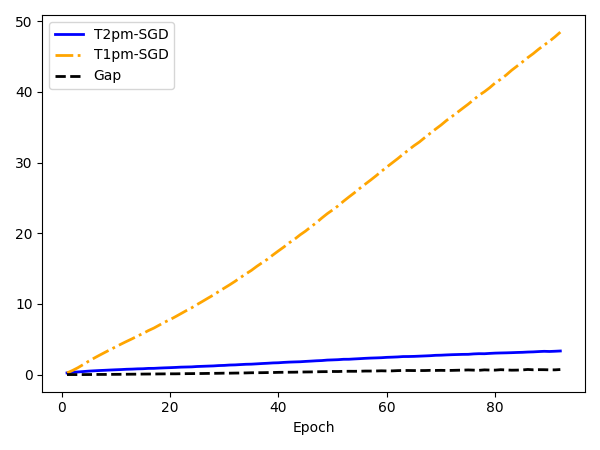}
		\caption{CIFAR-10}
		\label{fig:sub2}
	\end{subfigure}
	\caption{Comparison of generalization error bounds: {\it T2pm-SGD} vs.~{\it T1pm-SGD} on MNIST (left panel) and CIFAR-10 (right panel). The empirical gap is defined as the difference between the training loss and the testing loss.}
	\label{fig1}
\end{figure}

Figure \ref{fig1} shows that the generalization error bound obtained using {\it T2pm-SGD} (\(\check{\bm \theta}_k\)) is consistently lower than that obtained using {\it T1pm-SGD} (\(\tilde{\bm \theta}_k\)) and more closely aligns with the shape of the error gap. In contrast, the bound for {\it T1pm-SGD} increases steadily with each epoch, while the bound for {\it T2pm-SGD} remains stable throughout the training process. This behavior aligns with the design characteristics of the two methods.

\begin{figure}[htbp]
	\centering
	\begin{subfigure}[b]{0.51\textwidth} 
		\centering
		\includegraphics[width=\textwidth]{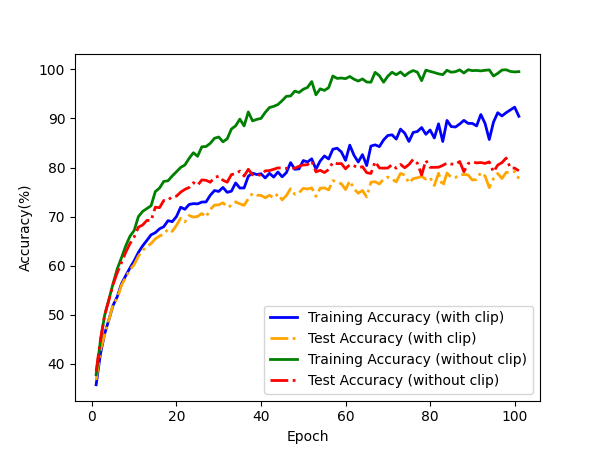}
		\caption{Accuracy}
		\label{fig:sub2}
	\end{subfigure}
	\hfill
	\begin{subfigure}[b]{0.46\textwidth} 
		\centering
		\includegraphics[width=\textwidth]{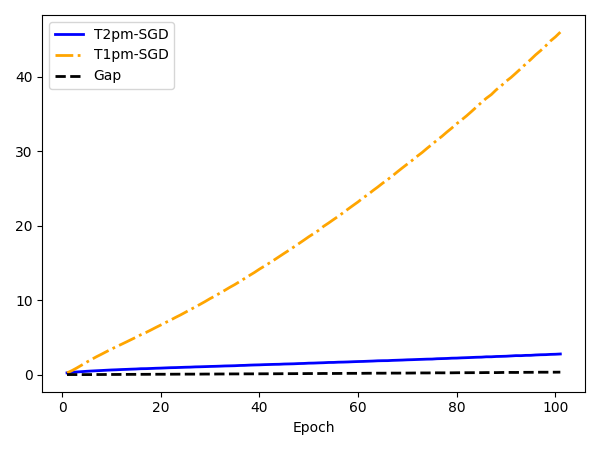}
		\caption{Bound for clipped SGD}
		\label{fig:sub1}
	\end{subfigure}
	\caption{Accuracy (left panel) and generalization error bounds (right panel) for clipped SGD during training on CIFAR-10.}
	\label{fig4}
\end{figure}

Next, we further examine the clipped SGD method, a widely used approach to prevent gradient explosion and improve model stability during training. In this experiment, we apply gradient clipping with a threshold of 5 to the CIFAR-10 dataset. As illustrated in Figure \ref{fig4}(a), gradient clipping effectively reduces overfitting, resulting in improved generalization performance. This highlights the effectiveness of gradient clipping in enhancing the robustness of the training process. Additionally, we evaluate the generalization error bounds for the clipped SGD method on CIFAR-10, derived using both {\it T1pm-SGD} and {\it T2pm-SGD}. As shown in Figure \ref{fig4}(b), the generalization error bound obtained with {\it T2pm-SGD} exhibits a slower rate of increase compared to that with {\it T1pm-SGD}, further demonstrating the robustness of the {\it T2pm-SGD} framework. These findings highlight the dual benefits of gradient clipping in controlling overfitting and achieving better model performance, while also validating the improved generalization bounds provided by {\it T2pm-SGD}.

\begin{figure}[htbp]
	\centering
	\begin{subfigure}[b]{0.243\textwidth} 
		\centering
		\includegraphics[width=\textwidth]{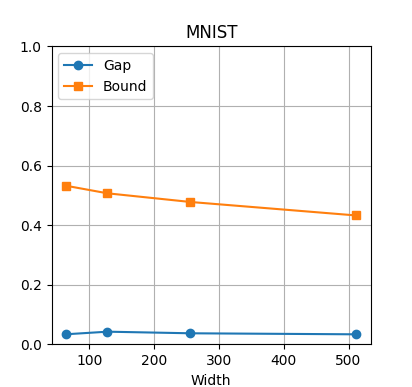}
		\caption{\small{$\sigma_{k}=0.005$}}
		\label{fig:sub2}
	\end{subfigure}
	\hfill
	\begin{subfigure}[b]{0.243\textwidth} 
		\centering
		\includegraphics[width=\textwidth]{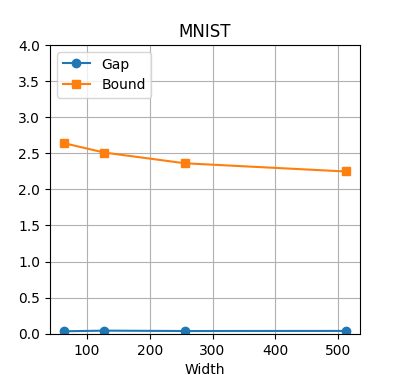}
		\caption{\small{$\sigma_{k}=0.01$}}
		\label{fig:sub1}
	\end{subfigure}
	\hfill
	\begin{subfigure}[b]{0.248\textwidth} 
		\centering
		\includegraphics[width=\textwidth]{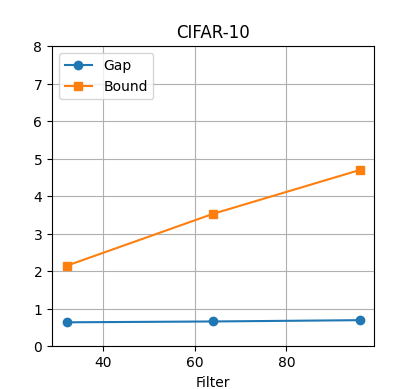}
		\caption{\small{$\sigma_{k}=0.005$}}
		\label{fig:sub3}
	\end{subfigure}
	\hfill
	\begin{subfigure}[b]{0.240\textwidth} 
		\centering
		\includegraphics[width=\textwidth]{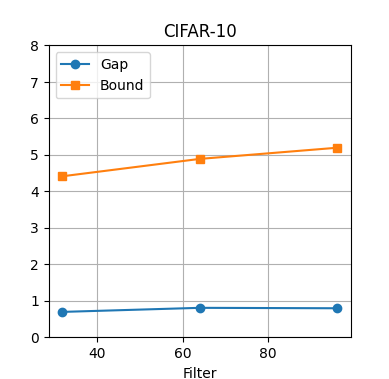}
		\caption{\small{$\sigma_{k}=0.001$}}
		\label{fig:sub4}
	\end{subfigure}	
	\caption{Generalization error bounds with {\it T2pm-SGD}: variation by widths on MNIST (Subfigures (a) and (b)) and by filters on CIFAR-10 (Subfigures (c) and (d)).}
	\label{fig3}
\end{figure}

Finally, we examine the impact of network hyperparameters on the generalization error bound derived using {\it T2pm-SGD}. Specifically, we vary the width of the MLP from 64 to 512 and the number of filters in the first layer of AlexNet from 32 to 96. The corresponding results are presented in Figure \ref{fig3}. 
The results in Figure \ref{fig3} show that the generalization gap remains relatively small and stable across different widths for MLP training on MNIST and filter sizes for AlexNet training on CIFAR-10, regardless of $\sigma_{k}$. Notably, the generalization error bound decreases as the width of the MLP increases, suggesting that wider networks improve the bound, while for AlexNet, larger filter sizes result in slightly looser bounds.

	\section{Conclusions}
	In this paper, we have investigated the stability and generalization properties of SGD in non-convex learning environments. By introducing \textit{T2pm-SGD}, we derived the tighter bounds on the generalization error for both bounded and sub-Gaussian loss functions. Our analysis decomposes the generalization error into two components: the trajectory term and the flatness term, each reflecting different aspects of the algorithm's behavior.
	
	For bounded loss functions, we achieved a significant improvement in the trajectory term, reducing it from \( O\big( (nb)^{-1/2} \big) \) to \( O(n^{-1}) \), independent of the batch size \( b \). This enhancement is particularly impactful in practical settings with small batch sizes, as it leads to better generalization without increasing computational resources. Moreover, we showed that by appropriately selecting the variance of the perturbed noise, the total generalization error bound can reach \( O(n^{-2/3}) \). For sub-Gaussian loss functions, we also provide a tighter generalization error bound compared to existing results in the literature. Our analysis reveals the interplay between noise variance, learning rate and batch size, providing insights into how these hyperparameters influence the generalization capabilities of SGD in non-convex settings.

	Future work may explore extending these results to other optimization algorithms and investigating the impact of various perturbation types. Additionally, further empirical studies on larger and more complex datasets could provide additional evidence of the effectiveness of the \textit{T2pm-SGD} in real-world applications. Investigating the theoretical underpinnings of why certain perturbation strategies yield better generalization could offer deeper insights into the learning dynamics of neural networks. This work advances the theoretical understanding of SGD's generalization properties in non-convex learning, providing valuable insights to researchers and practitioners for improving model performance and developing algorithms with superior generalization capabilities.
	
	\section{Technique Details}
	\subsection{Lemmas}	
	
	Define random vectors $Z=X+Y$ and $Z'=X'+Y'$, where $X,X',Y,Y'$ are $d$-dimensional random vectors. Let $P_{Z}$ and $P_{Z'}$ represent the density functions of $Z$ and $Z'$, respectively. The joint density functions of $(X,Y)$ and $(X',Y')$ are denoted by $P_{X,Y}$ and $P_{X',Y'}$, respectively.
	The conditional density of $Y$ given $X$ is denoted by $P_{Y|X}$, and $P_X$ represents the marginal distribution of $X$.
	Define $P_{X'}$ and $P_{Y'|X'}$ accordingly.

	\begin{lemma}\label{lem01}
		The KL divergence between the density functions of $Z$ and $Z'$ satisfies
		\begin{eqnarray*}
			\begin{split}
				D_{KL}\big(P_{Z'}\|P_{Z}\big)&\leq  D_{KL}\big(P_{X'}\|P_{X}\big)+\int D_{KL}\big(P_{Y'|X'=x}\|P_{Y|X=x}\big)P_{X'}(x)dx. 
			\end{split}
		\end{eqnarray*}
		
	\end{lemma}
	
	\begin{proof}
		Denote $g(x,z)=P_{X,Y}(x,z-x)/P_{Z}(z)$ and $G(u)=u\log(u),~u>0.$ Since $Z'=X'+Y'$, we have $P_{Z'}(z)=\int P_{X',Y'}(x,z-x)dx,$ 
		\begin{equation*}
			\begin{split}
				\int \dfrac{P_{X',Y'}(x,z-x)}{P_{X,Y}(x,z-x)}g(x,z)dx
				=\dfrac{P_{Z'}(z)}{P_{Z}(z)},
			\end{split}
		\end{equation*}
		and
		\begin{equation*}		
		\begin{split}			&D_{KL}\big(P_{Z'}\|P_{Z}\big)\\
			=&\int P_{Z'}(z) \log\Big(\dfrac{P_{Z'}(z)}{P_{Z}(z)}\Big)dz\\
			=&\int P_{Z}(z)G\bigg(\int \dfrac{P_{X',Y'}(x,z-x)}{P_{X,Y}(x,z-x)}g(x,z)dx\bigg)dz\\
			\leq& \int P_{Z}(z)\int G\Big(\dfrac{P_{X',Y'}(x,z-x)}{P_{X,Y}(x,z-x)}\Big)g(x,z)dxdz\\
			=& \int P_{Z}(z)\int \dfrac{P_{X,Y}(x,z-x)}{P_{Z}(z)} \dfrac{P_{X',Y'}(x,z-x)}{P_{X,Y}(x,z-x)} \log\Big(\dfrac{P_{X',Y'}(x,z-x)}{P_{X,Y}(x,z-x)}\Big)dxdz\\
			=& \int\int P_{X',Y'}(x,z-x) \log\Big(\dfrac{P_{X',Y'}(x,z-x)}{P_{X,Y}(x,z-x)}\Big)dxdz\\
			=&\int P_{X'}(x)\log\Big(\dfrac{P_{X'}(x)}{P_X(x)}\Big)dx
			+\int P_{X'}(x)P_{Y'|X'}(z-x|x)\log\Big(\dfrac{P_{Y'|X'}(z-x|x)}{P_{Y|X}(z-x|x)}\Big)dz dx\\			=&D_{KL}\big(P_{X'}\|P_{X}\big)+\int D_{KL}\big(P_{Y'|X'=x}\|P_{Y|X=x}\big)P_{X'}(x)dx,
		\end{split}
		\end{equation*}
		where the inequality comes from the Jensen's inequality and the convexity of 
		$G(u)$ for $u>0$.
		This completes the proof.	
	\end{proof}

	\subsection{Proof of Theorem \ref{th-sub-1.2}}
	\begin{proof}
		Define $L_{Jk}=1$ if the set $J$ intersects with the $k^{th}$ batch, and $L_{Jk}=0$ otherwise. Let $\bm L_J=(L_{J1},\ldots,L_{Jk})$. Then we have
		\[\check{ge}_k=\mathbb E_{S_{J^c},\bm L_J}\Big[\mathbb E_{S_J, S_J'|S_{J^c},\bm L_J}  \dfrac{1}{|J|}\displaystyle \sum_{i\in J}[f(\bm X_i, \check{\bm \theta}_k')-f(\bm X_i, \check{\bm \theta}_k)]\Big].\]
		Given that the loss function $f(\bm X,\bm \theta)$ is $R$-sub-Gaussian, the term $\dfrac{1}{|J|}\displaystyle \sum_{i\in J}f(\bm X_i, \check{\bm \theta}_k)$ is $\dfrac{R}{\sqrt{|J|}}$-sub-Gaussian, as illustrated in Lemma 1 of \citet{Pensia2018}. So, we have
		\begin{eqnarray*}
			\Big|\mathbb E_{S_J, S_J'|S_{J^c},\bm L_J}\dfrac{1}{|J|}\displaystyle \sum_{i\in J}[f(\bm X_i, \check{\bm \theta}_k')-f(\bm X_i, \check{\bm \theta}_k)]\Big|
			\leq \sqrt{2R^2I^{S_{J^c},\bm L_J}(\check{\bm \theta}_k;S_J)/|J|},
		\end{eqnarray*}
		where $I^{S_{J^c},\bm L_J}(\check{\bm \theta}_k;S_J)$ denotes the mutual information between $\check{\bm \theta}_k$ and $S_J$ given $S_{J^c}$ and $\bm L_J$.
		
		Next, we aim to establish an upper bound for $I^{S_{J^c},\bm L_J}(\check{\bm \theta}_k;S_J)$. Define $\check{\bm \theta}_{k-1:k}=(\check{\bm \theta}_{k-1},\check{\bm \theta}_k).$ Then, we obtain 
		\begin{equation*}
			I^{S_{J^c},\bm L_J}(\check{\bm \theta}_k;S_J)\leq I^{S_{J^c},\bm L_J}(\check{\bm \theta}_{k-1:k};S_J).
		\end{equation*}
		Applying the chain rule of mutual information, we find that
		\begin{equation}\label{prof4-1.1}
			I^{S_{J^c},\bm L_J}(\check{\bm \theta}_k;S_J)\leq I^{S_{J^c},\bm L_J}(\check{\bm \theta}_{k-1};S_J)+I^{S_{J^c},\bm L_J}(\check{\bm \theta}_{k};S_J|\check{\bm \theta}_{k-1}).
		\end{equation}	
		For simplicity, let $\check{\kappa}_k=P_{\check{\bm \theta}_{k},S_J|\check{\bm \theta}_{k-1},\bm L_J}=\int P_{\check{\bm \theta}_{k},S_J,\bm \epsilon_{k-1}|\check{\bm \theta}_{k-1},\bm L_J} d\bm \epsilon_{k-1}$ represent the conditional joint density function of $(\check{\bm \theta}_{k},S_J)$ given $\check{\bm \theta}_{k-1}$ and  $\bm L_J$, and define $\check{\kappa}'_k=P_{\check{\bm \theta}_{k}|\check{\bm \theta}_{k-1},\bm L_J}P_{S_J|\check{\bm \theta}_{k-1},\bm L_J}$. Then we have
		\begin{equation}\label{prof4-1.12}
			I^{S_{J^c},\bm L_J}(\check{\bm \theta}_{k};S_J|\check{\bm \theta}_{k-1})=\int \check{\kappa}_k(v)\ln \dfrac{\check{\kappa}_k(v)}{\check{\kappa}'_k(v)}dv=\int \check{\kappa}'_k(v)Q\Big(\dfrac{\check{\kappa}_k(v)}{\check{\kappa}'_k(v)}\Big)dv.
		\end{equation}	
		where $Q(x)=x\ln x$ for $x>0$.
		Because $S_J$ and $\bm \epsilon_{k-1}$ are independent, we have 
		\[\check{\kappa}'_k=\int P_{\check{\bm \theta}_{k},\bm \epsilon_{k-1}|\check{\bm \theta}_{k-1},\bm L_J}P_{S_J|\check{\bm \theta}_{k-1},\bm L_J}d\bm \epsilon_{k-1}.\]
		For simplicity of notation, we define $\check{q}_k=P_{\check{\bm \theta}_{k},S_J,\bm \epsilon_{k-1}|\check{\bm \theta}_{k-1},\bm L_J},$ $\check{q}'_k=P_{\check{\bm \theta}_{k},\bm \epsilon_{k-1}|\check{\bm \theta}_{k-1},\bm L_J}P_{S_J|\check{\bm \theta}_{k-1},\bm L_J}$ and $\check{\rho}_k=\check{q}'_k/\check{\kappa}'_k$. Then, we have
		\[\int \check{\rho}_k(v,\bm \epsilon_{k-1}) \dfrac{\check{q}_k(v,\bm \epsilon_{k-1})}{\check{q}'_k(v,\bm \epsilon_{k-1})}d\epsilon_{k-1}=\dfrac{\check{\kappa}_k(v)}{\check{\kappa}'_k(v)}.\]
		Plugging it into equation (\ref{prof4-1.12}), we obtain
		\[I^{S_{J^c},\bm L_J}(\check{\bm \theta}_{k};S_J|\check{\bm \theta}_{k-1})=\int \check{\kappa}'_k(v)Q\Big(\int \check{\rho}_k(v,\bm \epsilon_{k-1}) \dfrac{\check{q}_k(v,\bm \epsilon_{k-1})}{\check{q}'_k(v,\bm \epsilon_{k-1})}d\epsilon_{k-1}\Big)dv.\]
		Since $Q(x)=x\ln x$ is convex on the interval $(0,+\infty)$, by the Jensen's inequality, we have 
		$$Q\Big(\int \check{\rho}_k(v,\bm \epsilon_{k-1}) \dfrac{\check{q}_k(v,\bm \epsilon_{k-1})}{\check{q}'_k(v,\bm \epsilon_{k-1})}d\epsilon_{k-1}\Big)\leq \int \check{\rho}_k(v,\bm \epsilon_{k-1}) Q\Big(\dfrac{\check{q}_k(v,\bm \epsilon_{k-1})}{\check{q}'_k(v,\bm \epsilon_{k-1})}\Big)d\epsilon_{k-1}$$ and 
		\begin{eqnarray*}
			\begin{split}
				I^{S_{J^c},\bm L_J}(\check{\bm \theta}_{k};S_J|\check{\bm \theta}_{k-1})
				&\leq \int \check{\kappa}'_k(v)\int \check{\rho}_k(v,\bm \epsilon_{k-1}) Q\Big(\dfrac{\check{q}_k(v,\bm \epsilon_{k-1})}{\check{q}'_k(v,\bm \epsilon_{k-1})}\Big)d\epsilon_{k-1}dv\\
				&=\int \int \check{q}_k(v,\bm \epsilon_{k-1})\ln\Big(\dfrac{\check{q}_k(v,\bm \epsilon_{k-1})}{\check{q}'_k(v,\bm \epsilon_{k-1})}\Big)d\epsilon_{k-1}dv.
			\end{split}
		\end{eqnarray*}
		Plugging $\check{q}_k$ and $\check{q}'_k$ into the above expression, we directly obtain
		\begin{eqnarray}
			\begin{split}\label{th1-18-1.2}
				&I^{S_{J^c},\bm L_J}(\check{\bm \theta}_{k};S_J|\check{\bm \theta}_{k-1})
				\\	\leq&~\int f(\bm \epsilon_{k-1}) \int f(\check{\bm \theta}_{k},S_J|\check{\bm \theta}_{k-1},\bm \epsilon_{k-1},\bm L_J)
				\ln \bigg(\dfrac{f(\check{\bm \theta}_{k},S_J|\check{\bm \theta}_{k-1},\bm \epsilon_{k-1},\bm L_J)}{f(\check{\bm \theta}_{k}|\check{\bm \theta}_{k-1},\bm \epsilon_{k-1},\bm L_J)P_{S_J|\check{\bm \theta}_{k-1},\bm L_J}}\bigg)dvd\epsilon_{k-1}\\
				=&~\mathbb E_{\bm \epsilon_{k-1}|S_{J^c},\bm L_J} I^{S_{J^c},\bm \epsilon_{k-1},\bm L_J}(\check{\bm \theta}_{k};S_J|\check{\bm \theta}_{k-1}).
			\end{split}
		\end{eqnarray}	
		Denote by $H^{S_{J^c},\bm \epsilon_{k-1},\bm L_J}(\cdot)$ the conditional entropy given $S_{J^c},$ $\bm \epsilon_{k-1}$ and $\bm L_J$. We have	\begin{eqnarray}\label{th1-18-1.4}		I^{S_{J^c},\bm \epsilon_{k-1},\bm L_J}(\check{\bm \theta}_{k};S_J|\check{\bm \theta}_{k-1})
			=H^{S_{J^c},\bm \epsilon_{k-1},\bm L_J}(\check{\bm \theta}_{k}|\check{\bm \theta}_{k-1})
			-H^{S_{J^c},\bm \epsilon_{k-1},\bm L_J}(\check{\bm \theta}_{k}|\check{\bm \theta}_{k-1},S_J).	\end{eqnarray}
		Note that $\check{\bm \theta}_{k}=\check{\bm \theta}_{k-1}-\bm \epsilon_{k-1}-\eta_{k} \bm g_{S,k}(\check{\bm \theta}_{k-1}-\bm \epsilon_{k-1})+\bm \epsilon_k$. Given $\bm \epsilon_{k-1}$, we can obtain,
		\begin{eqnarray*}
			H^{S_{J^c},\bm \epsilon_{k-1},\bm L_J}(\check{\bm \theta}_{k}|\check{\bm \theta}_{k-1})
			=H^{S_{J^c},\bm \epsilon_{k-1},\bm L_J}\{\check{\bm \theta}_{k-1}-\eta_{k} \bm g_{S,k}(\check{\bm \theta}_{k-1}-\bm \epsilon_{k-1})+\bm \epsilon_k|\check{\bm \theta}_{k-1}\}.	
		\end{eqnarray*}
		Recalling that $\bm g_{S,k}(\check{\bm \theta}_{k-1}-\bm \epsilon_{k-1})=\dfrac{1}{b}\sum\limits_{i\in M_k}\nabla f(\bm X_i, \check{\bm \theta}_{k-1}-\bm \epsilon_{k-1})$,
		we have
		\[\bm g_{S,k}(\check{\bm \theta}_{k-1}-\bm \epsilon_{k-1})=\bm g_{S_J,k}(\check{\bm \theta}_{k-1}-\bm \epsilon_{k-1})+\bm g_{S_{J^c},k}(\check{\bm \theta}_{k-1}-\bm \epsilon_{k-1}),\]
		where $\bm g_{S_J,k}(\check{\bm \theta}_{k-1}-\bm \epsilon_{k-1})=\dfrac{1}{b}\sum\limits_{i\in M_k\cap S_J}\nabla f(\bm X_i, \check{\bm \theta}_{k-1}-\bm \epsilon_{k-1}).$
		
		Given $S_{J^c}$, $\bm \epsilon_{k-1}$ and $ \check{\bm \theta}_{k-1}$, $\bm g_{S_{J^c},k}(\check{\bm \theta}_{k-1}-\bm \epsilon_{k-1})$ is deterministic and does not influence the conditional entropy $H^{S_{J^c},\bm \epsilon_{k-1},\bm L_J}(\check{\bm \theta}_{k}|\check{\bm \theta}_{k-1})$. Therefore, the entropy can be reformulated as
		\[H^{S_{J^c},\bm \epsilon_{k-1},\bm L_J}(\check{\bm \theta}_{k}|\check{\bm \theta}_{k-1})
		=H^{S_{J^c},\bm \epsilon_{k-1},\bm L_J}\{\eta_k\bm g_{S_J,k}(\check{\bm \theta}_{k-1}-\bm \epsilon_{k-1})+\bm \epsilon_k|\check{\bm \theta}_{k-1}\}.\]	
		
		Next, we will compute the bound for $H^{S_{J^c},\bm \epsilon_{k-1},\bm L_J}(\check{\bm \theta}_{k}|\check{\bm \theta}_{k-1}).$ We first calculate the following expectation, 
		\begin{eqnarray*}
			\begin{split}
				&\mathbb E\big[\|\eta_k\bm g_{S_J,k}(\check{\bm \theta}_{k-1}-\bm \epsilon_{k-1})+\bm \epsilon_k\|^2|\check{\bm \theta}_{k-1},S_{J^c},\bm \epsilon_{k-1},\bm{L}\big]\\
				=&~ d\sigma^2_k+\eta_{k}^2\mathrm{tr}\big[\mathbb V^{S_{J^c},\bm \epsilon_{k-1},\bm L_J}\big(\bm g_{S_J,k}(\check{\bm \theta}_{k-1}-\bm \epsilon_{k-1})|\check{\bm \theta}_{k-1}\big)\big],
			\end{split}
		\end{eqnarray*}
		where $\mathbb V^{S_{J^c},\bm \epsilon_{k-1},\bm L_J}$ is the conditional variance given $S_{J^c},\bm \epsilon_{k-1}$ and $\bm{L}$. 
		According to the proof of Lemma 5 in \citet{Pensia2018}, among all $d$-dimensional random variables with a fixed second-order moment, the Gaussian distribution has the highest entropy. This implies that
		\begin{eqnarray}\label{prof4-1.19}
			H^{S_{J^c},\bm \epsilon_{k-1},\bm L_J}(\check{\bm \theta}_{k}|\check{\bm \theta}_{k-1})
			\leq \dfrac{d}{2}\mathbb \log\Big[\dfrac{2\pi e }{d}\Big(d\sigma^2_k+\eta_{k}^2\mathrm{tr}\big[\mathbb V^{S_{J^c},\bm \epsilon_{k-1},\bm L_J}\big(\bm g_{S_J,k}(\check{\bm \theta}_{k-1}-\bm \epsilon_{k-1})|\check{\bm \theta}_{k-1}\big)\big]\Big)\Big].
		\end{eqnarray}
		
		To bound the expression in (\ref{th1-18-1.4}), we need to compute the second term on the right-hand side $H^{S_{J^c},\bm \epsilon_{k-1},\bm L_J}(\check{\bm \theta}_{k}|\check{\bm \theta}_{k-1}, S_J)$. We start by defining it as
		\begin{eqnarray*}
			H^{S_{J^c},\bm \epsilon_{k-1},\bm L_J}(\check{\bm \theta}_{k}|\check{\bm \theta}_{k-1},S_J)
			=H^{S_{J^c},\bm \epsilon_{k-1},\bm L_J}\{\check{\bm \theta}_{k-1}-\eta_{k} \bm g_{S,k}(\check{\bm \theta}_{k-1}-\bm \epsilon_{k-1})+\bm \epsilon_k|\check{\bm \theta}_{k-1},S_J\}.	
		\end{eqnarray*}
		In the above expression, given $S_{J^c}, S_J,\check{\bm \theta}_{k-1},\bm \epsilon_{k-1}$, and $\bm L_J$, $\check{\bm \theta}_{k-1}-\eta_{k} \bm g_{S,k}(\check{\bm \theta}_{k-1}-\bm \epsilon_{k-1})$ is a constant, and does not affect the entropy calculation. Taking into account that $\bm \epsilon_k\sim N(0,\sigma^2_k\bm I_d),$ we have
		\begin{eqnarray}\label{prof4-1.2}
			H^{S_{J^c},\bm \epsilon_{k-1},\bm L_J}(\check{\bm \theta}_{k}|\check{\bm \theta}_{k-1},S_J)=H^{S_{J^c},\bm \epsilon_{k-1},\bm L_J}\{\bm \epsilon_k|\check{\bm \theta}_{k-1},S_J\}=\dfrac{d}{2}\log(2\pi e \sigma^2_k).
		\end{eqnarray}
		Combining (\ref{th1-18-1.2})-(\ref{prof4-1.2}) and $\bm \theta_{k-1}=\check{\bm \theta}_{k-1}-\bm \epsilon_{k-1},$ we obtain the upper bound of the mutual information,
		\begin{eqnarray*}
			\begin{split}
				I^{S_{J^c},\bm L_J}(\check{\bm \theta}_{k};S_J|\check{\bm \theta}_{k-1})
				&\leq\dfrac{d}{2} \mathbb E_{\bm \epsilon_{k-1}|S_{J^c},\bm L_J}\log\bigg(1+\eta^2_k\dfrac{\mathrm{tr}\big[\mathbb V^{S_{J^c},\bm \epsilon_{k-1},\bm L_J}\big(\bm g_{S_J,k}(\check{\bm \theta}_{k-1}-\bm \epsilon_{k-1})|\check{\bm \theta}_{k-1}\big)\big]}{d \sigma^2_k}\bigg).
			\end{split}
		\end{eqnarray*}
		By (\ref*{prof4-1.1}), we have
		\begin{eqnarray*}
			\begin{split}
				I^{S_{J^c},\bm L_J}(\check{\bm \theta}_{k};S_J)&\leq I^{S_{J^c},\bm L_J}(\check{\bm \theta}_{k-1};S_J)\\
				&~~~+\dfrac{d}{2} \mathbb E_{\bm \epsilon_{k-1}|S_{J^c},\bm L_J}\log\bigg(1+\eta^2_k\dfrac{\mathrm{tr}\big[\mathbb V^{S_{J^c},\bm \epsilon_{k-1},\bm L_J}\big(\bm g_{S_J,k}(\check{\bm \theta}_{k-1}-\bm \epsilon_{k-1})|\check{\bm \theta}_{k-1}\big)\big]}{d \sigma^2_k}\bigg).
			\end{split}
		\end{eqnarray*}
		By the induction, we have
		\[I^{S_{J^c},\bm L_J}(\check{\bm \theta}_{k};S_J)\leq \dfrac{d}{2} \displaystyle \sum_{s=1}^{k}\mathbb E_{\bm \epsilon_{s-1}|S_{J^c},\bm L_J}\log\bigg(1+\eta^2_s\dfrac{\mathrm{tr}\big[\mathbb V^{S_{J^c},\bm \epsilon_{s-1},\bm L_J}\big(\bm g_{S_J,s}(\check{\bm \theta}_{s-1}-\bm \epsilon_{s-1})|\check{\bm \theta}_{s-1}\big)\big]}{d \sigma^2_s}\bigg).\]
		Since $J$ is arbitrary, we obtain the upper bound of the generalization error as follows
		\begin{eqnarray*}
			\begin{split}
				|ge_k|&\leq
				\min_{J}\Bigg\{|\xi_{k}(\bm \theta_k,\bm \epsilon_k)|\\
				&~~~+\mathbb E_{S_{J^c},\bm L_J}\sqrt{\dfrac{R^2d}{|J|} \displaystyle \sum_{s=1}^{k}\mathbb E_{\bm \epsilon_{s-1}|S_{J^c},\bm L_J}\log\bigg(1+\eta^2_s\dfrac{\mathrm{tr}\big[\mathbb V^{S_{J^c},\bm \epsilon_{s-1},\bm L_J}\big(\bm g_{S_J,s}(\check{\bm \theta}_{s-1}-\bm \epsilon_{s-1})|\check{\bm \theta}_{s-1}\big)\big]}{d \sigma^2_s}\bigg)}\Bigg\}.
			\end{split}
		\end{eqnarray*}
		
		If $J=\{j\}$, it follows that $m=1$, $\bm g_{S_{\{j\}},s}(\bm \theta_{s-1})=\dfrac{L_{\{j\}s}}{b}\nabla f(\bm X_j, \bm \theta_{s-1})$ and the result could be rewritten as
		\begin{eqnarray*}
			\begin{split}
				&|ge_k|\\
				&
				\leq |\xi_{k}(\bm \theta_k,\bm \epsilon_k)|\\
				&~~~+\mathbb E_{S_{\{j\}^c},\bm L_{\{j\}}}\sqrt{R^2d \displaystyle \sum_{s=1}^{k}\mathbb E_{\bm \epsilon_{s-1}|S_{\{j\}^c},\bm L_{\{j\}}}\log\bigg(1+L_{\{j\}s}\eta^2_s\dfrac{\mathrm{tr}\big[\mathbb V^{S_{\{j\}^c},\bm \epsilon_{s-1},\bm L_{\{j\}}}\big(\nabla f(\bm X_j, \bm \theta_{s-1})|\check{\bm \theta}_{s-1}\big)\big]}{d \sigma^2_s}\bigg)}\\
				&= |\xi_{k}(\bm \theta_k,\bm \epsilon_k)|\\
				&~~~+\mathbb E_{S_{\{j\}^c},\bm L_{\{j\}}}\sqrt{R^2d \displaystyle \sum_{s=1}^{k}\mathbb E_{\bm \epsilon_{s-1}|S_{\{j\}^c},\bm L_{\{j\}}}L_{\{j\}s}\log\bigg(1+\eta^2_s\dfrac{\mathrm{tr}\big[\mathbb V^{S_{\{j\}^c},\bm \epsilon_{s-1},\bm L_{\{j\}}}\big(\nabla f(\bm X_j, \bm \theta_{s-1})|\check{\bm \theta}_{s-1}\big)\big]}{d \sigma^2_s}\bigg)}.
			\end{split}
		\end{eqnarray*}
		Then the proof is finished.
		
	\end{proof}
	
	\subsection{Proof of Theorem \ref{th-sub-3.2}} 
	
	\begin{proof} 
We follow a similar approach to the proof of Theorem \ref{th-sub-1.2}, replacing the variance matrix \(\sigma_k I_d\) of the noise \(\bm \epsilon_k\) with the general covariance matrix \(\Sigma_k\). It then follows that
		\begin{eqnarray}\label{2prof4-1.19}
			\begin{split}
			&H^{S_{J^c},\bm \epsilon_{k-1},\bm L_J}(\check{\bm \theta}_{k}|\check{\bm \theta}_{k-1})\\
			\leq& \dfrac{d}{2}\mathbb \log\Big[\dfrac{2\pi e }{d}\Big(d|\Sigma_k|^{1/d}+\eta_{k}^2\mathrm{tr}\big[\mathbb V^{S_{J^c},\bm \epsilon_{k-1},\bm L_J}\big(\bm g_{S_J,k}(\check{\bm \theta}_{k-1}-\bm \epsilon_{k-1})|\check{\bm \theta}_{k-1}\big)\big]\Big)\Big].
			\end{split}
		\end{eqnarray}
		Since $\epsilon_k\sim N(0,\Sigma_k),$ we have	
		\begin{eqnarray}\label{2prof4-1.2}
			H^{S_{J^c},\bm \epsilon_{k-1},\bm L_J}(\check{\bm \theta}_{k}|\check{\bm \theta}_{k-1},S_J)=H^{S_{J^c},\bm \epsilon_{k-1},\bm L_J}\{\bm \epsilon_k|\check{\bm \theta}_{k-1},S_J\}=\frac{d}{2}\log(2\pi e |\Sigma_k|^{1/d}).
		\end{eqnarray}
		Combining (\ref{th1-18-1.4}), (\ref{2prof4-1.2}) and $\bm \theta_{k-1}=\check{\bm \theta}_{k-1}-\bm \epsilon_{k-1},$ we obtain the upper bound of the mutual information between \(\check{\bm \theta}_{k}\) and \(S_J\) conditioned on \(\check{\bm \theta}_{k-1}\),
		\begin{eqnarray*}
			\begin{split}
				I^{S_{J^c},\bm L_J}(\check{\bm \theta}_{k};S_J|\check{\bm \theta}_{k-1})
				&\leq\dfrac{d}{2} \mathbb E_{\bm \epsilon_{k-1}|S_{J^c},\bm L_J}\log\bigg(1+\eta^2_k\dfrac{\mathrm{tr}\big[\mathbb V^{S_{J^c},\bm \epsilon_{k-1},\bm L_J}\big(\bm g_{S_J,k}(\check{\bm \theta}_{k-1}-\bm \epsilon_{k-1})|\check{\bm \theta}_{k-1}\big)\big]}{d |\Sigma_k|^{1/d}}\bigg).
			\end{split}
		\end{eqnarray*}
		By (\ref*{prof4-1.1}), we have
		\begin{eqnarray*}
			\begin{split}
				I^{S_{J^c},\bm L_J}(\check{\bm \theta}_{k};S_J)&\leq I^{S_{J^c},\bm L_J}(\check{\bm \theta}_{k-1};S_J)\\
				&~~~+\dfrac{d}{2} \mathbb E_{\bm \epsilon_{k-1}|S_{J^c},\bm L_J}\log\bigg(1+\eta^2_k\dfrac{\mathrm{tr}\big[\mathbb V^{S_{J^c},\bm \epsilon_{k-1},\bm L_J}\big(\bm g_{S_J,k}(\check{\bm \theta}_{k-1}-\bm \epsilon_{k-1})|\check{\bm \theta}_{k-1}\big)\big]}{d|\Sigma_k|^{1/d}}\bigg).
			\end{split}
		\end{eqnarray*}
		By the induction, we have
		\[I^{S_{J^c},\bm L_J}(\check{\bm \theta}_{k};S_J)\leq \dfrac{d}{2} \displaystyle \sum_{s=1}^{k}\mathbb E_{\bm \epsilon_{s-1}|S_{J^c},\bm L_J}\log\bigg(1+\eta^2_s\dfrac{\mathrm{tr}\big[\mathbb V^{S_{J^c},\bm \epsilon_{s-1},\bm L_J}\big(\bm g_{S_J,s}(\check{\bm \theta}_{s-1}-\bm \epsilon_{s-1})|\check{\bm \theta}_{s-1}\big)\big]}{d |\Sigma_s|^{1/d}}\bigg).\]
		Since the subset \(J\) is arbitrary, we can select it to minimize the bound on the generalization error. Thus, we obtain
		\begin{eqnarray*}
			\begin{split}
				|ge_k|&\leq
				\min_{J}\Bigg\{|\xi_{k}(\bm \theta_k,\bm \epsilon_k)|\\
				&~~~+\mathbb E_{S_{J^c},\bm L_J}\sqrt{\dfrac{R^2d}{|J|} \displaystyle \sum_{s=1}^{k}\mathbb E_{\bm \epsilon_{s-1}|S_{J^c},\bm L_J}\log\bigg(1+\eta^2_s\dfrac{\mathrm{tr}\big[\mathbb V^{S_{J^c},\bm \epsilon_{s-1},\bm L_J}\big(\bm g_{S_J,s}(\check{\bm \theta}_{s-1}-\bm \epsilon_{s-1})|\check{\bm \theta}_{s-1}\big)\big]}{d |\Sigma_s|^{1/d}}\bigg)}\Bigg\}.
			\end{split}
		\end{eqnarray*}
		Then the proof is finished.
		
	\end{proof}

	\subsection{Proof of Theorem \ref{th-sub-1}} 
	\begin{proof}
		Recall that \[\check{ge}_k=\mathbb E_{S, S'}\Big[\dfrac{1}{|J|}\displaystyle \sum_{i\in J}\int f(\bm X_i,\bm \theta) \big(P_{\check{\bm{\theta}}'_k|S, S'}(\bm \theta)-P_{\check{\bm{\theta}}_k|S, S'}(\bm \theta)\big)d\bm \theta\Big].\]
		By Condition 2.1, $f(\bm X,\bm \theta)$ is $c_0$-bounded, and it follows that
		\[|\check{ge}_k|\leq 2c_0\mathbb E_{S, S'}D_{TV}\big(P_{\check{\bm{\theta}}'_k|S, S'}\|P_{\check{\bm{\theta}}_k|S, S'}\big).\]
		To calculate $\check{ge}_k$, we need to assess the similarity between the densities $P_{\check{\bm{\theta}}'_k|S, S'}$ and $P_{\check{\bm{\theta}}_k|S, S'}$, which can be measured using the total variation distance and the Kullback-Leibler divergence.
		
		Recall that $\check{\bm \theta}_{k}=\bm \theta_k+\bm \epsilon_k$ and $\check{\bm \theta}'_{k}=\bm \theta'_k+\bm \epsilon'_k.$ It follows that
		\[\check{\bm \theta}_{k}=\bm \theta_{k-1}-\eta_{k} \bm g_{S,k}(\bm \theta_{k-1})+\bm \epsilon_k=\check{\bm \theta}_{k-1}-\bm \epsilon_{k-1}-\eta_{k} \bm g_{S,k}(\check{\bm \theta}_{k-1}-\bm \epsilon_{k-1})+\bm \epsilon_k\] and \[\check{\bm \theta}'_{k}=\bm \theta'_{k-1}-\eta_{k} \bm g_{S',k}(\bm \theta'_{k-1})+\bm \epsilon'_k=\check{\bm \theta}'_{k-1}-\bm \epsilon'_{k-1}-\eta_{k} \bm g_{S',k}(\check{\bm \theta}'_{k-1}-\bm \epsilon'_{k-1})+\bm \epsilon'_k.\]
		Denote a $d$-dimensional random vector
		$\bm h_{k}=\bm \theta_{k-1} -\eta_{k} \bm g_{S,k}(\bm \theta_{k-1} ),$ $\mathring{\bm \iota}_{t}=s_{k}W_{t}$ and
		$\bm \iota_{t}'=t[\bm g_{S',k}(\bm \theta'_{k-1})-\bm g_{S,k}(\bm \theta'_{k-1})]+s_{k}W'_t$, with $s^2_{k}= \sigma^2_k/\eta_k$.
		Let the density functions of $\bm h_{k}$ and $\bm h'_{k}$ conditional $S$ and $S'$ be $P_{\bm h_{k}|S, S'}$ and $P_{\bm h'_{k}|S, S'}$, respectively.
		
		When $t=\eta_k,$ we find $\mathring{\bm \iota}_{\eta_k}=\bm \epsilon_k$ and $\bm \iota_{\eta_k}'=\eta_k[\bm g_{S',k}(\bm \theta'_{k-1})-\bm g_{S,k}(\bm \theta'_{k-1})]+\bm \epsilon'_k.$
		As a result, the update rule can be reformulated as
		$\check{\bm \theta}_{k}=\bm h_k+\mathring{\bm \iota}_{\eta_k}$ and
		$\check{\bm \theta}'_{k}=\bm h'_{k}+\bm \iota_{\eta_k}'$.  Let $P_{\mathring{\bm \iota}_{t}|\bm h_{k}, S, S'}$ denote the density function of $\mathring{\bm \iota}_{t}$ conditional on $\bm h_{k}$, $S$ and $S'$, and $P_{\bm \iota_{t}'|\bm h'_{k}, S, S'}$ represent the density functions of $\bm \iota_{t}'$ conditional on $\bm h'_{k}, S$ and $S'$.
		By Lemma 1, we obtain the KL divergence between $\check{\bm \theta}_{k}$ and $\check{\bm \theta}'_{k}$ as follows,
		\begin{eqnarray}\label{negho-01}
			\begin{split}
				D_{KL}(P_{\check{\bm{\theta}}'_k|S, S'}\|P_{\check{\bm{\theta}}_k|S, S'})&\leq D_{KL}(P_{\bm h'_{k}|S, S'}\|P_{\bm h_{k}|S, S'})\\
				&~~~+\int D_{KL}\big( P_{\bm \iota_{\eta_k}'|\bm h_{k}=h, S, S'}\| P_{\mathring{\bm \iota}_{\eta_k}|\bm h'_{k}=h, S, S'} \big)P_{\bm h'_{k}| S, S'}(h)dh.
			\end{split}
		\end{eqnarray}	
		Denote $ P_{\check{\bm \omega}_k|S, S'}$ and $ P_{\check{\bm \omega}'_k|S, S'}$ as the conditional density functions of $\check{\bm \omega}_k=\eta_{k} \bm g_{S,k}(\check{\bm \theta}_{k-1}-\bm \epsilon_{k-1})$ and $\check{\bm \omega}'_k=\eta_{k} \bm g_{S,k}(\check{\bm \theta}'_{k-1}-\bm \epsilon'_{k-1})$ conditional on $S$ and $S'$, respectively. By $\bm h_{k}=\check{\bm \theta}_{k-1}-\bm \epsilon_{k-1}-\eta_{k} \bm g_{S,k}(\check{\bm \theta}_{k-1}-\bm \epsilon_{k-1}),$ we find that the density $P_{\bm h_{k}|S, S'}$ equals the convolution of the following three density functions: $P_{\check{\bm{\theta}}_{k-1}|S, S'}$, $ P_{\check{\bm \omega}_k|S, S'}$ and $N(0,\sigma_k^2\bm I_d)$. Combining with the data processing inequality from \cite{Ahlswede1976} and Lemma 1, we have the following inequality
		\begin{eqnarray*}
			\begin{split}
				D_{KL}(P_{\bm h'_{k}|S, S'}\|P_{\bm h_{k}|S, S'})&\leq 	D_{KL}(P_{\check{\bm{\theta}}'_{k-1}|S, S'}\|P_{\check{\bm{\theta}}_{k-1}|S, S'})\\
				&~~~+\int D_{KL}( P_{\check{\bm \omega}'_k|\check{\bm \theta}_{k-1}=\theta,S, S'}\| P_{\check{\bm \omega}_k|\check{\bm \theta}_{k-1}'=\theta,S, S'})P_{\check{\bm \theta}_{k-1}'}(\theta)d\theta.
			\end{split}
		\end{eqnarray*}
		Given $\check{\bm \theta}_{k-1}=\check{\bm \theta}'_{k-1}=\theta,$ the density functions of $\check{\bm \omega}_k=\eta_{k} \bm g_{S,k}(\check{\bm \theta}_{k-1}-\bm \epsilon_{k-1})$ and $\check{\bm \omega}'_k=\eta_{k} \bm g_{S,k}(\check{\bm \theta}'_{k-1}-\bm \epsilon'_{k-1})$ are the same, 
		which implies that
		\[D_{KL}( P_{\check{\bm \omega}'_k|\check{\bm \theta}_{k-1}=\theta,S, S'}\| P_{\check{\bm \omega}_k|\check{\bm \theta}_{k-1}'=\theta,S, S'})=0.\] 
		It follows that
		\[D_{KL}(P_{\bm h'_{k}|S, S'}\|P_{\bm h_{k}|S, S'})\leq 	D_{KL}(P_{\check{\bm{\theta}}'_{k-1}|S, S'}\|P_{\check{\bm{\theta}}_{k-1}|S, S'}).\]
		Combining it with (\ref{negho-01}),  we have	
		\begin{eqnarray}\label{negho-02}
			\begin{split}
				D_{KL}(P_{\check{\bm{\theta}}'_k|S, S'}\|P_{\check{\bm{\theta}}_k|S, S'})&\leq D_{KL}(P_{\check{\bm{\theta}}'_{k-1}|S, S'}\|P_{\check{\bm{\theta}}_{k-1}|S, S'})\\
				&~~~+\int D_{KL}\big( P_{\bm \iota_{\eta_k}'|\bm h_{k}=h, S, S'}\| P_{\mathring{\bm \iota}_{\eta_k}|\bm h'_{k}=h, S, S'} \big)P_{\bm h'_{k}| S, S'}(h)dh.
			\end{split}
		\end{eqnarray}	
		In the subsequent analysis, we aim to establish the upper bound for $$\int D_{KL}\big( P_{\bm \iota_{\eta_k}'|\bm h_{k}=h, S, S'}\| P_{\mathring{\bm \iota}_{\eta_k}|\bm h'_{k}=h, S, S'} \big)P_{\bm h'_{k}| S, S'}(h)dh.$$
		Let $a\odot b$ represent the inner product of vectors $a$ and $b$. At the initial time $t=0$, we have $\mathring{\bm \iota}_{0}=\bm \iota'_{0}=0.$ For any time $t\leq \eta_k$, the following continuous-time Langevin equation holds,
		\[d\mathring{\bm \iota}_{t}=s_{k}dW_t~\text{and}~d\bm \iota_{t}'=[\bm g_{S',k}(\bm \theta'_{k-1})-\bm g_{S,k}(\bm \theta'_{k-1})]+s_{k}dW'_t.\] According to \citet{Mou2017} and \citet{gyongy1986},
		the densities $ P_{\mathring{\bm \iota}_{t}|\bm h_{k}, S, S'}$ and $P_{\bm \iota_{t}'|\bm h'_{k}, S, S'}$ satisfy the Fokker-Planck equations
		\begin{equation}\label{th1-eq1}
			\dfrac{\partial}{\partial t} P_{\mathring{\bm \iota}_{t}|\bm h_{k}, S, S'}=\frac{s^2_{k}}{2}\bigtriangleup  P_{\mathring{\bm \iota}_{t}|\bm h_{k}, S, S'}~\text{and}~\dfrac{\partial}{\partial t}P_{\bm \iota_{t}'|\bm h'_{k}, S, S'}=-\nabla\odot(F_{J,k}P_{\bm \iota_{t}'|\bm h'_{k}, S, S'})+\frac{s^2_{k}}{2}\bigtriangleup P_{\bm \iota_{t}'|\bm h'_{k}, S, S'},
		\end{equation}
		where $\nabla=(\dfrac{\partial}{\partial w_1},\ldots, \dfrac{\partial}{\partial w_d})^{\top}$ and $F_{J,k}(w)=\mathbb E[\bm g_{S',k}(\bm \theta'_{k-1})-\bm g_{S,k}(\bm \theta'_{k-1})|\bm \iota_{t}'=w, \bm h'_{k},S, S'].$ 
		By the definition of $\bm g_{S,k}$ and $\bm g_{S',k}$, we have
		\[F_{J,k}(\bm \iota_{t}')=\dfrac{1}{b}\displaystyle \sum_{j\in J} \mathbb E\big(\bm I_{k, j}[\nabla f(\bm X_{j}', \bm \theta'_{k-1})-\mathbb E\nabla f(\bm X_{j}, \bm \theta'_{k-1})]|\bm \iota_{t}', \bm h'_{k},S, S'\big),\]
		where $\bm I_{k, j}=1$ indicates the inclusion of the $j^{th}$ sample in the $k^{th}$ batch, and $\bm I_{k, j}=0$ otherwise. Since the batch selection is independent of $ \bm \theta'_{k-1}$, $\bm I_{k, j}$ is independent of $ \bm \theta'_{k-1}$, with an expected value of $\mathbb E\bm I_{k, j}=b/n$. It follows that
		\begin{equation*}
			F_{J,k}(\bm \iota_{t}')=\dfrac{1}{n}\displaystyle \sum_{j\in J} \mathbb E\big[\big(\nabla f(\bm X_{j}', \bm \theta'_{k-1})-\nabla f(\bm X_{j}, \bm \theta'_{k-1})\big)\big|\bm \iota_{t}', \bm h'_{k}, S, S'\big].
		\end{equation*}
		Since $\bm X_{j}$ is independent with $\bm \theta'_{k-1}$ for any $j\in J$, given $\bm h'_{k}$, $S$ and $S'$, we have
		\begin{equation}\label{th1-def1}
			F_{J,k}(\bm \iota_{t}')=\dfrac{1}{n}\displaystyle \sum_{j\in J} \mathbb E\big[\big(\nabla f(\bm X_{j}', \bm \theta'_{k-1})-\mathbb E\nabla f(\bm X_{j}, \bm \theta'_{k-1})\big)\big|\bm \iota_{t}', \bm h'_{k}, S, S'\big]=H_{J,k}(\bm \iota_{t}')/n.
		\end{equation}
		and
		\begin{align*}
			&~~\dfrac{d}{dt}D_{KL}(P_{\bm \iota_{t}'|\bm h'_{k}, S, S'}\| P_{\mathring{\bm \iota}_{t}|\bm h_{k}, S, S'})\\
			&=\int \dfrac{d P_{\bm \iota_{t}'|\bm h'_{k}, S, S'}(u)}{dt}\left[\log\dfrac{P_{\bm \iota_{t}'|\bm h'_{k}, S, S'}(u)}{ P_{\mathring{\bm \iota}_{t}|\bm h_{k}, S, S'}(u)}-1\right]du-\int  \dfrac{d P_{\mathring{\bm \iota}_{t}|\bm h_{k}, S, S'}(u)}{d_{t}}\dfrac{P_{\bm \iota_{t}'|\bm h'_{k}, S, S'}(u)}{ P_{\mathring{\bm \iota}_{t}|\bm h_{k}, S, S'}(u)}du\\
			&=\int \nabla\odot\left[-F_{J,k}(\bm \iota_{t}')P_{\bm \iota_{t}'|\bm h'_{k}, S, S'}(u)+\frac{s^2_{k}}{2}\nabla P_{\bm \iota_{t}'|\bm h'_{k}, S, S'}(u)\right]\left[\log\dfrac{P_{\bm \iota_{t}'|\bm h'_{k}, S, S'}(u)}{ P_{\mathring{\bm \iota}_{t}|\bm h_{k}, S, S'}(u)}-1\right]du\\
			&~~~-\int\nabla\odot\left[\frac{s^2_{k}}{2}\nabla P_{\mathring{\bm \iota}_{t}|\bm h_{k}, S, S'}(u)\right]\dfrac{P_{\bm \iota_{t}'|\bm h'_{k}, S, S'}(u)}{ P_{\mathring{\bm \iota}_{t}|\bm h_{k}, S, S'}(u)}du\\
			&=-\int \left[-F_{J,k}(\bm \iota_{t}')P_{\bm \iota_{t}'|\bm h'_{k}, S, S'}(u)+\dfrac{s^2_{k}}{2}\nabla P_{\bm \iota_{t}'|\bm h'_{k}, S, S'}(u)\right]\odot\nabla\log\dfrac{P_{\bm \iota_{t}'|\bm h'_{k}, S, S'}(u)}{ P_{\mathring{\bm \iota}_{t}|\bm h_{k}, S, S'}(u)}du\\
			&~~~+\int\left[\frac{s^2_{k}}{2}\nabla P_{\mathring{\bm \iota}_{t}|\bm h_{k}, S, S'}(u)\right]\dfrac{P_{\bm \iota_{t}'|\bm h'_{k}, S, S'}(u)}{ P_{\mathring{\bm \iota}_{t}|\bm h_{k}, S, S'}(u)}\odot\nabla\log\dfrac{P_{\bm \iota_{t}'|\bm h'_{k}, S, S'}(u)}{ P_{\mathring{\bm \iota}_{t}|\bm h_{k}, S, S'}(u)}du\\
			&=-\int  \nabla^{\top}\log\dfrac{P_{\bm \iota_{t}'|\bm h'_{k}, S, S'}(u)}{ P_{\mathring{\bm \iota}_{t}|\bm h_{k}, S, S'}(u)}\dfrac{s^2_{k}}{2}\nabla\log\dfrac{P_{\bm \iota_{t}'|\bm h'_{k}, S, S'}(u)}{ P_{\mathring{\bm \iota}_{t}|\bm h_{k}, S, S'}(u)}P_{\bm \iota_{t}'|\bm h'_{k}, S, S'}(u)du\\
			&~~~+\int F_{J,k}(\bm \iota_{t}')\odot\left(\nabla\log\dfrac{P_{\bm \iota_{t}'|\bm h'_{k}, S, S'}(u)}{ P_{\mathring{\bm \iota}_{t}|\bm h_{k}, S, S'}(u)}\right)P_{\bm \iota_{t}'|\bm h'_{k}, S, S'}(u)du\\
			&\leq -\dfrac{s^2_{k}}{2}\int  \bigg\|\nabla\log\dfrac{P_{\bm \iota_{t}'|\bm h'_{k}, S, S'}(u)}{ P_{\mathring{\bm \iota}_{t}|\bm h_{k}, S, S'}(u)}\bigg\|^2P_{\bm \iota_{t}'|\bm h'_{k}, S, S'}(u)du\\
			&~~~+\int F_{J,k}(\bm \iota_{t}')\odot\left(\nabla\log\dfrac{P_{\bm \iota_{t}'|\bm h'_{k}, S, S'}(u)}{ P_{\mathring{\bm \iota}_{t}|\bm h_{k}, S, S'}(u)}\right)P_{\bm \iota_{t}'|\bm h'_{k}, S, S'}(u)du\\
			&\leq -\dfrac{s^2_{k}}{4}\int  \bigg\|\nabla\log\dfrac{P_{\bm \iota_{t}'|\bm h'_{k}, S, S'}(u)}{ P_{\mathring{\bm \iota}_{t}|\bm h_{k}, S, S'}(u)}\bigg\|^2P_{\bm \iota_{t}'|\bm h'_{k}, S, S'}(u)du+\dfrac{1}{s^2_{k}}\int\|F_{J,k}(u)\|^2P_{\bm \iota_{t}'|\bm h'_{k}, S, S'}(u)du.
		\end{align*}
		Since $ P_{\mathring{\bm \iota}_{t}|\bm h_{k}, S, S'}$ follows a normal distribution $N(0,s_k^2t)$, according to Theorem 1 in \citet{Markowich2000}, $ P_{\mathring{\bm \iota}_{t}|\bm h_{k}, S, S'}$ adheres to a logarithmic
		Sobolev inequality with a constant of $1/(s^2_{k}t),$ that is,
		\begin{equation*}
			\int  \bigg\|\nabla\log\dfrac{P_{\bm \iota_{t}'|\bm h'_{k}, S, S'}(u)}{ P_{\mathring{\bm \iota}_{t}|\bm h_{k}, S, S'}(u)}\bigg\|^2P_{\bm \iota_{t}'|\bm h'_{k}, S, S'}(u)du\geq \dfrac{2}{ts^2_{k}}D_{KL}(P_{\bm \iota_{t}'|\bm h'_{k}, S, S'}\| P_{\mathring{\bm \iota}_{t}|\bm h_{k}, S, S'}).
		\end{equation*}
		So, for any $t\in(0,1)$ and $\alpha_k\in(0,1)$, we have
		\begin{equation*}
			\begin{split}
				&~~\dfrac{d}{dt}D_{KL}(P_{\bm \iota_{t}'|\bm h'_{k}, S, S'}\| P_{\mathring{\bm \iota}_{t}|\bm h_{k}, S, S'})\\
				&\leq -\dfrac{1}{2t}D_{KL}(P_{\bm \iota_{t}'|\bm h'_{k}, S, S'}\| P_{\mathring{\bm \iota}_{t}|\bm h_{k}, S, S'})+\dfrac{1}{s^2_{k}}\int\|F_{J,k}(u)\|^2P_{\bm \iota_{t}'|\bm h'_{k}, S, S'}(u)du\\
				&\leq -\dfrac{1}{2t^{\alpha_k}}D_{KL}(P_{\bm \iota_{t}'|\bm h'_{k}, S, S'}\| P_{\mathring{\bm \iota}_{t}|\bm h_{k}, S, S'})+\dfrac{1}{s^2_{k}}\int\|F_{J,k}(u)\|^2P_{\bm \iota_{t}'|\bm h'_{k}, S, S'}(u)du\\
				&\leq  -\dfrac{1}{2t^{\alpha_k}}D_{KL}(P_{\bm \iota_{t}'|\bm h'_{k}, S, S'}\| P_{\mathring{\bm \iota}_{t}|\bm h_{k}, S, S'})+\dfrac{c_1^2\eta_k}{n^2\sigma^2_{k}},
			\end{split}
		\end{equation*}	
		where the last inequality comes from $\|F_{J,k}(u)\|^2= \|H_{J,k}(u)\|^2/n^2\leq c_1^2/n^2$.
		It follows that
		\begin{eqnarray*}
			D_{KL}(P_{\bm \iota_{t}'|\bm h'_{k}, S, S'}\| P_{\mathring{\bm \iota}_{t}|\bm h_{k}, S, S'})
			&\leq  \Big(e^{-\frac{1}{2(1-\alpha_k)}t^{1-\alpha_k}}\int_{0}^{t}e^{\frac{1}{2(1-\alpha_k)}u^{1-\alpha_k}}du\Big) \times \dfrac{c_1^2\eta_k}{n^2\sigma^2_{k}}.
		\end{eqnarray*} 
		Setting $t=\eta_k$, we get
		\[D_{KL}(P_{\bm \iota_{\eta_k}'|\bm h'_{k}, S, S'}\| P_{\mathring{\bm \iota}_{\eta_k}|\bm h'_{k}, S, S'})
		\leq \dfrac{\delta_kc_1^2\eta_k}{n^2\sigma^2_{k}},\]
		where $\delta_k=e^{-\frac{1}{2(1-\alpha_k)}\eta_k^{1-\alpha_k}}\int_{0}^{\eta_k }e^{\frac{1}{2(1-\alpha_k)}u^{1-\alpha_k}}du.$
		Combining this with (\ref{negho-02}), we have
		\begin{eqnarray*}
			\begin{split}
				D_{KL}(P_{\check{\bm{\theta}}'_k|S, S'}\|P_{\check{\bm{\theta}}_k|S, S'})&\leq D_{KL}(P_{\check{\bm{\theta}}'_{k-1}|S, S'}\|P_{\check{\bm{\theta}}_{k-1}|S, S'})+\dfrac{\delta_kc_1^2\eta_k}{n^2\sigma^2_{k}}\\
				&\leq D_{KL}(P_{\check{\bm{\theta}}'_{k-2}|S, S'}\|P_{\check{\bm{\theta}}_{k-2}|S, S'})+\dfrac{\delta_kc_1^2\eta_k}{n^2\sigma^2_{k}}+\dfrac{\delta_{k-1}c_1^2\eta_{k-1}}{n^2\sigma^2_{k-1}}\\
				&\leq \ldots\\
				&\leq \displaystyle \sum_{j=1}^k \dfrac{\delta_jc_1^2\eta_j}{n^2\sigma^2_{j}}.
			\end{split}
		\end{eqnarray*}	
		Combining it with the Pinsker's inequality $D_{TV}\leq \sqrt{D_{KL}/2}$, we have
		\[D_{TV}(P_{\check{\bm{\theta}}'_k|S, S'}\|P_{\check{\bm{\theta}}_k|S, S'})\leq  \sqrt{\displaystyle \sum_{j=1}^k \dfrac{\delta_jc_1^2\eta_j}{2n^2\sigma^2_{j}}}.\]
		Finally, we arrive at the following bound for the generalization error,
		\[|ge_k|\leq |\xi_{k}(\bm \theta_k,\bm \epsilon_k)|+\dfrac{c_0c_1}{n} \sqrt{\displaystyle \sum_{j=1}^k\dfrac{2\delta_j\eta_j}{\sigma^2_{j}}}.\]
		The proof is finished.	
	\end{proof}

	\subsection{Proof of Theorem \ref{th-sub-2}}
	\begin{proof}
		Define $\mathring{\bm \iota}_{t}=S_{k}W_{t}$ and
		\[\bm \iota_{t}'=t[\bm g_{S',k}(\bm \theta'_{k-1})-\bm g_{S,k}(\bm \theta'_{k-1})]+S_{k}W'_t,\] where $S_{k}=\bm \Sigma_k/\eta_k.$ Then, we have
		\begin{equation*}
			\begin{split}
				&d\mathring{\bm \iota}_{t}=S_{k}dW_t~\text{and}~
				d\bm \iota_{t}'=\bm g_{S',k}(\bm \theta'_{k-1})-\bm g_{S,k}(\bm \theta'_{k-1})+S_{k}dW'_t.
			\end{split}
		\end{equation*}
		Let $ P_{\mathring{\bm \iota}_{t}|\bm h_{k}, S, S'}$ and $P_{\bm \iota_{t}'|\bm h'_{k}, S, S'}$ represent the density functions of $\mathring{\bm \iota}_{t}$ and $\bm \iota_{t}'$ given $\bm h'_{k}$, $S$ and $S'$, respectively. According to (\ref{th1-eq1}), these two density functions adhere to the subsequent differential equation,
		\begin{equation}\label{th1-eq11}
			\begin{split}
				&\dfrac{\partial}{\partial t} P_{\mathring{\bm \iota}_{t}|\bm h_{k}, S, S'}=\frac{S^2_{k}}{2}\bigtriangleup  P_{\mathring{\bm \iota}_{t}|\bm h_{k}, S, S'}~\text{and}~\dfrac{\partial}{\partial t}P_{\bm \iota_{t}'|\bm h'_{k}, S, S'}=-\nabla\odot\big(F_{J,k}(\bm \iota_{t}')P_{\bm \iota_{t}'|\bm h'_{k}, S, S'}\big)+\frac{S^2_{k}}{2}\bigtriangleup P_{\bm \iota_{t}'|\bm h'_{k}, S, S'},
			\end{split}
		\end{equation}
		where  $F_{J,k}(w)=\dfrac{1}{n}\displaystyle \sum_{j\in J} \mathbb E\big[\big(\nabla f(\bm X_{j}', \bm \theta'_{k-1})-\mathbb E\nabla f(\bm X_{j}, \bm \theta'_{k-1})\big)\big|\bm \iota_{t}'=w,\bm h'_{k}, S, S'\big].$
		It follows that,
		\begin{equation*}
			\begin{split}
				&~~~\dfrac{d}{dt}D_{KL}(P_{\bm \iota_{t}'|\bm h'_{k}, S, S'}\| P_{\mathring{\bm \iota}_{t}|\bm h_{k}, S, S'})\\
				&=\int \dfrac{d P_{\bm \iota_{t}'|\bm h'_{k}, S, S'}(u)}{dt}\left[\log\dfrac{P_{\bm \iota_{t}'|\bm h'_{k}, S, S'}(u)}{ P_{\mathring{\bm \iota}_{t}|\bm h_{k}, S, S'}(u)}-1\right]du-\int  \dfrac{d P_{\mathring{\bm \iota}_{t}|\bm h_{k}, S, S'}(u)}{d_{t}}\dfrac{P_{\bm \iota_{t}'|\bm h'_{k}, S, S'}(u)}{ P_{\mathring{\bm \iota}_{t}|\bm h_{k}, S, S'}(u)}du\\
				&=\int \nabla\odot\left[-F_{J,k}(\bm \iota_{t}')P_{\bm \iota_{t}'|\bm h'_{k}, S, S'}(u)+\frac{S^2_{k}}{2}\nabla P_{\bm \iota_{t}'|\bm h'_{k}, S, S'}(u)\right]\left[\log\dfrac{P_{\bm \iota_{t}'|\bm h'_{k}, S, S'}(u)}{ P_{\mathring{\bm \iota}_{t}|\bm h_{k}, S, S'}(u)}-1\right]du\\
				&~~~-\int\nabla\odot\left[\frac{S^2_{k}}{2}\nabla P_{\mathring{\bm \iota}_{t}|\bm h_{k}, S, S'}(u)\right]\dfrac{P_{\bm \iota_{t}'|\bm h'_{k}, S, S'}(u)}{ P_{\mathring{\bm \iota}_{t}|\bm h_{k}, S, S'}(u)}du\\
				&=-\int \left[-F_{J,k}(\bm \iota_{t}')P_{\bm \iota_{t}'|\bm h'_{k}, S, S'}(u)+\dfrac{S^2_{k}}{2}\nabla P_{\bm \iota_{t}'|\bm h'_{k}, S, S'}(u)\right]\odot\nabla\log\dfrac{P_{\bm \iota_{t}'|\bm h'_{k}, S, S'}(u)}{ P_{\mathring{\bm \iota}_{t}|\bm h_{k}, S, S'}(u)}du\\
				&~~~+\int\left[\frac{S^2_{k}}{2}\nabla P_{\mathring{\bm \iota}_{t}|\bm h_{k}, S, S'}(u)\right]\dfrac{P_{\bm \iota_{t}'|\bm h'_{k}, S, S'}(u)}{ P_{\mathring{\bm \iota}_{t}|\bm h_{k}, S, S'}(u)}\odot\nabla\log\dfrac{P_{\bm \iota_{t}'|\bm h'_{k}, S, S'}(u)}{ P_{\mathring{\bm \iota}_{t}|\bm h_{k}, S, S'}(u)}du\\
				&=-\int  \nabla^{\top}\log\dfrac{P_{\bm \iota_{t}'|\bm h'_{k}, S, S'}(u)}{ P_{\mathring{\bm \iota}_{t}|\bm h_{k}, S, S'}(u)}\dfrac{S^2_{k}}{2}\nabla\log\dfrac{P_{\bm \iota_{t}'|\bm h'_{k}, S, S'}(u)}{ P_{\mathring{\bm \iota}_{t}|\bm h_{k}, S, S'}(u)}P_{\bm \iota_{t}'|\bm h'_{k}, S, S'}(u)du\\
				&~~~+\int F_{J,k}(\bm \iota_{t}')\odot\left(\nabla\log\dfrac{P_{\bm \iota_{t}'|\bm h'_{k}, S, S'}(u)}{ P_{\mathring{\bm \iota}_{t}|\bm h_{k}, S, S'}(u)}\right)P_{\bm \iota_{t}'|\bm h'_{k}, S, S'}(u)du\\
				&\leq -\dfrac{\lambda_{k}}{2}\int  \bigg\|\nabla\log\dfrac{P_{\bm \iota_{t}'|\bm h'_{k}, S, S'}(u)}{ P_{\mathring{\bm \iota}_{t}|\bm h_{k}, S, S'}(u)}\bigg\|^2P_{\bm \iota_{t}'|\bm h'_{k}, S, S'}(u)du+\int F_{J,k}(\bm \iota_{t}')\odot\left(\nabla\log\dfrac{P_{\bm \iota_{t}'|\bm h'_{k}, S, S'}(u)}{ P_{\mathring{\bm \iota}_{t}|\bm h_{k}, S, S'}(u)}\right)P_{\bm \iota_{t}'|\bm h'_{k}, S, S'}(u)du\\
				&\leq -\dfrac{\lambda_{k}}{4}\int  \bigg\|\nabla\log\dfrac{P_{\bm \iota_{t}'|\bm h'_{k}, S, S'}(u)}{ P_{\mathring{\bm \iota}_{t}|\bm h_{k}, S, S'}(u)}\bigg\|^2P_{\bm \iota_{t}'|\bm h'_{k}, S, S'}(u)du+\dfrac{c_1^2\eta_k}{n^2\underline{\lambda}_{k}}.
			\end{split}
		\end{equation*}
		According to Theorem 1 in \citet{Markowich2000}, 
		the density function $ P_{\mathring{\bm \iota}_{t}|\bm h_{k}, S, S'}\sim N(0,t\Sigma_k)$ satisfies a logarithmic
		Sobolev inequality with a constant of $\eta_k/(\overline{\lambda}_{k}t),$ that is
		\begin{equation*}
			\int  \bigg\|\nabla\log\dfrac{P_{\bm \iota_{t}'|\bm h'_{k}, S, S'}(u)}{ P_{\mathring{\bm \iota}_{t}|\bm h_{k}, S, S'}(u)}\bigg\|^2P_{\bm \iota_{t}'|\bm h'_{k}, S, S'}(u)du\geq \dfrac{\eta_k}{t\overline{\lambda}_{k}}D_{KL}(P_{\bm \iota_{t}'|\bm h'_{k}, S, S'}\| P_{\mathring{\bm \iota}_{t}|\bm h_{k}, S, S'}).
		\end{equation*}
		Consequently, for any $t\in(0,1)$ and $\alpha_k\in(0,1)$, we have
		\begin{equation*}
			\begin{split}
				&~~~\dfrac{d}{dt}D_{KL}(P_{\bm \iota_{t}'|\bm h'_{k}, S, S'}\| P_{\mathring{\bm \iota}_{t}|\bm h_{k}, S, S'})\\
				&\leq -\dfrac{\underline{\lambda}_{k}}{2t\overline{\lambda}_{k}}D_{KL}(P_{\bm \iota_{t}'|\bm h'_{k}, S, S'}\| P_{\mathring{\bm \iota}_{t}|\bm h_{k}, S, S'})+\dfrac{c_1^2\eta_k}{n^2\underline{\lambda}_{k}}\\
				&\leq -\dfrac{\underline{\lambda}_{k}}{2\overline{\lambda}_{k}}\dfrac{1}{t^{\alpha_k}}D_{KL}(P_{\bm \iota_{t}'|\bm h'_{k}, S, S'}\| P_{\mathring{\bm \iota}_{t}|\bm h_{k}, S, S'})+\dfrac{c_1^2\eta_k}{n^2\underline{\lambda}_{k}}.
			\end{split}
		\end{equation*}	
		It follows that
		\begin{eqnarray*}
			D_{KL}(P_{\bm \iota_{t}'|\bm h'_{k}, S, S'}\| P_{\mathring{\bm \iota}_{t}|\bm h_{k}, S, S'})
			&\leq \big(e^{-\frac{1}{2(1-\alpha_k)}t^{1-\alpha_k}}\int_{0}^{t}e^{\frac{1}{2(1-\alpha_k)}u^{1-\alpha_k}}du\big)\times \dfrac{c_1^2\eta_k}{n^2\underline{\lambda}_{k}}.
		\end{eqnarray*} 
		By setting $t=\eta_k$, 
		$$D_{KL}(P_{\bm \iota_{\eta_k}'|\bm h'_{k}, S, S'}\| P_{\mathring{\bm \iota}_{\eta_k}|\bm h'_{k}, S, S'})\leq \dfrac{\zeta_kc_1^2\eta_k}{n^2\underline{\lambda}_{k}}$$ 
		with $\zeta_k=e^{-\frac{\underline{\lambda}_{k}}{2\overline{\lambda}_{k}(1-\alpha_k)}\eta_k ^{1-\alpha_k}}\int_{0}^{\eta_k }e^{\frac{\underline{\lambda}_{k}}{2\overline{\lambda}_{k}(1-\alpha_k)}u^{1-\alpha_k}}du$. 
		Combining it with (\ref{negho-02}), we have
		\begin{equation*}
			\begin{split}
				D_{KL}(P_{\check{\bm{\theta}}'_k|S, S'}\|P_{\check{\bm{\theta}}_k|S, S'})\leq D_{KL}(P_{\check{\bm{\theta}}'_{k-1}|S, S'}\|P_{\check{\bm{\theta}}_{k-1}|S, S'})+ \dfrac{\zeta_kc_1^2\eta_k}{n^2\underline{\lambda}_{k}}.
			\end{split}
		\end{equation*}
		By the induction, we ascertain that $D_{KL}(\check{\pi}'_{k|S, S', R},\check{\pi}_{k|S, S', R}) \leq \displaystyle \sum_{s=1}^k\dfrac{\zeta_sc_1^2\eta_s}{n^2\underline{\lambda}_{s}}$. Furthermore, integrating it with Pinsker's inequality, we obtain
		\[D_{TV}(P_{\check{\bm{\theta}}'_k|S, S'}\|P_{\check{\bm{\theta}}_k|S, S'})\leq  \sqrt{\displaystyle \sum_{s=1}^k\dfrac{\zeta_sc_1^2\eta_s}{2n^2\underline{\lambda}_{s}}}.\]	
		Then we derive the following bound		
		\[|ge_k|\leq |\xi_{k}(\bm \theta_k,\bm \epsilon_k)|+ c_0 \sqrt{\displaystyle \sum_{s=1}^k\dfrac{2\zeta_sc_1^2\eta_s}{n^2\underline{\lambda}_{s}}}.\]
		Since $F_{J,k}(\bm \iota_{t}')=H_{J,k}(\bm \iota_{t}')/n,$ 
		\begin{eqnarray*}
			|ge_k|\leq |\xi_{k}(\bm \theta_k,\bm \epsilon_k)|+\dfrac{c_0c_1}{n}\sqrt{\displaystyle \sum_{s=1}^k\dfrac{2\zeta_s\eta_s}{\underline{\lambda}_{s}}},
		\end{eqnarray*}
		Then the proof is finished.	
	\end{proof} 
	
\subsection{Proof of Theorem \ref{2th-sub-1.2}}
\begin{proof}
This proof is divided into two parts. In the first part, we establish the generalized error bound for \(\epsilon_{s} \sim N(0, \sigma_s^2 \bm{I}_d)\). In the second part, we extend this result to the case where \(\bm{\epsilon}_s \sim N(0, \bm{\Sigma}_s)\).\\

\noindent\textit{Part 1.} Based on the proof of Theorem \ref{th-sub-1.2}, we obtain the following inequality
\begin{eqnarray*}
		\begin{split}
			|ge_k|&\leq
			\min_{J}\Bigg\{|\xi_{k}(\bm \theta_k,\bm \epsilon_k)|\\
			+&\mathbb E_{S_{J^c},\bm L_J}\sqrt{\dfrac{R^2d}{|J|} \displaystyle \sum_{s=1}^{k}\mathbb E_{\bm \epsilon_{s-1}|S_{J^c},\bm L_J}\log\bigg(1+\eta^2_s\dfrac{\mathrm{tr}\big[\mathbb V^{S_{J^c},\bm \epsilon_{s-1},\bm L_J}\big(\tilde{g}_{S_J,s}(\check{\bm \theta}_{s-1}-\bm \epsilon_{s-1})|\check{\bm \theta}_{s-1}\big)\big]}{d \sigma^2_s}\bigg)}\Bigg\}.
		\end{split}
	\end{eqnarray*}
If \( L_{Js} = 0 \), then \( \mathrm{tr}\left[ \mathbb{V}^{S_{J^c}, \bm \epsilon_{s-1}, \bm L_J} \left( \tilde{\bm g}_{S_J,s} \left( \check{\bm \theta}_{s-1} - \bm \epsilon_{s-1} \right) \mid \check{\bm \theta}_{s-1} \right) \right] = 0 \). In clipped SGD, we impose a constraint on the gradient $\bm \tilde{\bm g}_{S_J,s}$ by limiting its magnitude to a constant $A$. 
By applying the inequality \( \mathbb{E}(\sqrt{X}) \leq \sqrt{\mathbb{E}(X)} \) for any non-negative random variable \( X \), along with \( P(L_{Js} = 1) = \frac{b}{n} \) and \( \|\bm \tilde{\bm g}_{S_J,s}\| \leq A \), we obtain 
\[
|ge_k| \leq |\xi_k(\bm \theta_k, \bm \epsilon_k)| + \sqrt{\frac{R^2d}{n} \sum_{s=1}^{k} \log \left( 1 + \frac{A^2 \eta_s^2}{\sigma_s^2} \right)}.
\]

\noindent\textit{Part 2.}
    From the proof of Theorem \ref{th-sub-3.2}, we obtain that
	\begin{eqnarray*}
		\begin{split}
			|ge_k|&\leq
			\min_{J}\Bigg\{|\xi_{k}(\bm \theta_k,\bm \epsilon_k)|\\
			&~~~+\mathbb E_{S_{J^c},\bm L_J}\sqrt{\dfrac{R^2d}{|J|} \displaystyle \sum_{s=1}^{k}\mathbb E_{\bm \epsilon_{s-1}|S_{J^c},\bm L_J}\log\bigg(1+\eta^2_s\dfrac{\mathrm{tr}\big[\mathbb V^{S_{J^c},\bm \epsilon_{s-1},\bm L_J}\big(\tilde{\bm g}_{S_J,s}(\check{\bm \theta}_{s-1}-\bm \epsilon_{s-1})|\check{\bm \theta}_{s-1}\big)\big]}{d |\Sigma_s|^{1/d}}\bigg)}\Bigg\}.
		\end{split}
	\end{eqnarray*}
    By applying a similar approach as in Part 1, and using the fact that \( \|\bm \tilde{\bm g}_{S_J,s}\| \leq A \), we obtain the following bound,
    \begin{eqnarray*}
		\begin{split}
			|ge_k|&\leq |\xi_{k}(\bm \theta_k,\bm \epsilon_k)|+\sqrt{\dfrac{R^2d}{n} \displaystyle \sum_{s=1}^{k}\log\bigg(1+\eta^2_s\dfrac{A^2}{ |\Sigma_s|^{1/d}}\bigg)}.
		\end{split}
	\end{eqnarray*}
	Then the proof is finished.  
\end{proof}

\subsection{Proof of Theorem \ref{th-c0-bound}}
\begin{proof}
This proof is divided into two parts. In the first part, we establish the generalized error bound under \( c_0 \)-bounded loss for the scenario where \( \epsilon_s \sim N(0, \sigma_s^2 \bm{I}_d) \). In the second part, we extend this result to the case where \( \bm{\epsilon}_s \sim N(0, \bm{\Sigma}_s) \).\\

\noindent\textit{Part 1}. Using the notations from the proof of Theorem \ref{th-sub-1} and by (\ref{negho-01}), we obtain the following inequality 
    \begin{eqnarray}\label{negho-02}
			\begin{split}
				D_{KL}(P_{\check{\bm{\theta}}'_k|S, S'}\|P_{\check{\bm{\theta}}_k|S, S'})&\leq D_{KL}(P_{\check{\bm{\theta}}'_{k-1}|S, S'}\|P_{\check{\bm{\theta}}_{k-1}|S, S'})\\
				&~~~+\int D_{KL}\big( P_{\bm \iota_{\eta_k}'|\bm h_{k}=h, S, S'}\| P_{\mathring{\bm \iota}_{\eta_k}|\bm h'_{k}=h, S, S'} \big)P_{\bm h'_{k}| S, S'}(h)dh.
			\end{split}
		\end{eqnarray}	
    For any $t\leq \eta_k$, the following continuous-time Langevin equation holds,
	\[d\mathring{\bm \iota}_{t}=s_{k}dW_t~\text{and}~d\bm \iota_{t}'=[\tilde{\bm g}_{S',k}(\bm \theta'_{k-1})-\tilde{\bm g}_{S,k}(\bm \theta'_{k-1})]+s_{k}dW'_t.\] 
    According to \citet{Mou2017} and \citet{gyongy1986},
	the densities $ P_{\mathring{\bm \iota}_{t}|\bm h_{k}, S, S'}$ and $P_{\bm \iota_{t}'|\bm h'_{k}, S, S'}$ satisfy the Fokker-Planck equations,
	\begin{equation}\label{2th1-eq1}
		\dfrac{\partial}{\partial t} P_{\mathring{\bm \iota}_{t}|\bm h_{k}, S, S'}=\frac{s^2_{k}}{2}\bigtriangleup  P_{\mathring{\bm \iota}_{t}|\bm h_{k}, S, S'}~\text{and}~\dfrac{\partial}{\partial t}P_{\bm \iota_{t}'|\bm h'_{k}, S, S'}=-\nabla\odot(F_{J,k}P_{\bm \iota_{t}'|\bm h'_{k}, S, S'})+\frac{s^2_{k}}{2}\bigtriangleup P_{\bm \iota_{t}'|\bm h'_{k}, S, S'},
	\end{equation}
	where $\nabla=(\dfrac{\partial}{\partial w_1},\ldots, \dfrac{\partial}{\partial w_d})^{\top}$ and $F_{J,k}(w)=\mathbb E[\tilde{\bm g}_{S',k}(\bm \theta'_{k-1})-\tilde{\bm g}_{S,k}(\bm \theta'_{k-1})|\bm \iota_{t}'=w,\bm h'_{k}, S, S'].$ 
	Denote $\bm I_{k, j}=1$ indicates the inclusion of the $j^{th}$ sample in the $k^{th}$ batch, and $\bm I_{k, j}=0$ otherwise. Since the batch selection is independent of $ \bm \theta'_{k-1}$, $\bm I_{k, j}$ is independent of $ \bm \theta'_{k-1}$, with an expected value of $\mathbb E(\bm I_{k, j})=b/n$. By the definition of $\tilde{\bm g}_{S,k}$ and $\tilde{\bm g}_{S',k}$, if $\bm I_{k, j}=0$, then $\tilde{\bm g}_{S',k}(\bm \theta'_{k-1})-\tilde{\bm g}_{S,k}(\bm \theta'_{k-1})=0$. So, we have 
	\begin{equation*}
		F_{J,k}(\bm \iota_{t}')=\dfrac{b}{n}\mathbb E[\tilde{\bm g}_{S',k}(\bm \theta'_{k-1})-\tilde{\bm g}_{S,k}(\bm \theta'_{k-1})|\bm I_{k, j}=1, \bm \iota_{t}'=w, \bm h'_{k},S, S'].
	\end{equation*}
	By the definition of clipped SGD, we have $\|\tilde{\bm g}_{S',k}\|\leq A$ and $\|\tilde{\bm g}_{S,k}\|\leq A$, which implies that
	\begin{equation}\label{2th1-def1}
		\|F_{J,k}(\bm \iota_{t}')\|\leq \dfrac{2Ab}{n}.
	\end{equation}
	Based on the proof of Theorem \ref{th-sub-1}, we obtain the following inequality
	\begin{align*}
		&\dfrac{d}{dt}D_{KL}(P_{\bm \iota_{t}'|\bm h'_{k}, S, S'}\| P_{\mathring{\bm \iota}_{t}|\bm h_{k}, S, S'})\\
		=&\int \dfrac{d P_{\bm \iota_{t}'|\bm h'_{k}, S, S'}(u)}{dt}\left[\log\dfrac{P_{\bm \iota_{t}'|\bm h'_{k}, S, S'}(u)}{ P_{\mathring{\bm \iota}_{t}|\bm h_{k}, S, S'}(u)}-1\right]du-\int  \dfrac{d P_{\mathring{\bm \iota}_{t}|\bm h_{k}, S, S'}(u)}{d_{t}}\dfrac{P_{\bm \iota_{t}'|\bm h'_{k}, S, S'}(u)}{ P_{\mathring{\bm \iota}_{t}|\bm h_{k}, S, S'}(u)}du\\
		\leq & -\dfrac{s^2_{k}}{4}\int  \bigg\|\nabla\log\dfrac{P_{\bm \iota_{t}'|\bm h'_{k}, S, S'}(u)}{ P_{\mathring{\bm \iota}_{t}|\bm h_{k}, S, S'}(u)}\bigg\|^2P_{\bm \iota_{t}'|\bm h'_{k}, S, S'}(u)du+\dfrac{1}{s^2_{k}}\int\|F_{J,k}(u)\|^2P_{\bm \iota_{t}'|\bm h'_{k}, S, S'}(u)du\\
		\leq & -\dfrac{s^2_{k}}{4}\int  \bigg\|\nabla\log\dfrac{P_{\bm \iota_{t}'|\bm h'_{k}, S, S'}(u)}{ P_{\mathring{\bm \iota}_{t}|\bm h_{k}, S, S'}(u)}\bigg\|^2P_{\bm \iota_{t}'|\bm h'_{k}, S, S'}(u)du+\dfrac{4A^2b^2}{n^2s^2_{k}}
	\end{align*}
	Since $ P_{\mathring{\bm \iota}_{t}|\bm h_{k}, S, S'}$ follows a normal distribution $N(0,s_k^2t)$, according to Theorem 1 in \citet{Markowich2000}, $ P_{\mathring{\bm \iota}_{t}|\bm h_{k}, S, S'}$ adheres to a logarithmic
	Sobolev inequality with a constant of $1/(s^2_{k}t),$ that is,
	\begin{equation*}
		\int  \bigg\|\nabla\log\dfrac{P_{\bm \iota_{t}'|\bm h'_{k}, S, S'}(u)}{ P_{\mathring{\bm \iota}_{t}|\bm h_{k}, S, S'}(u)}\bigg\|^2P_{\bm \iota_{t}'|\bm h'_{k}, S, S'}(u)du\geq \dfrac{2}{ts^2_{k}}D_{KL}(P_{\bm \iota_{t}'|\bm h'_{k}, S, S'}\| P_{\mathring{\bm \iota}_{t}|\bm h_{k}, S, S'}).
	\end{equation*}
	Thus, for any $t\in(0,1)$ and $\alpha_k\in(0,1)$, we have
	\begin{equation*}
		\begin{split}
			\dfrac{d}{dt}D_{KL}(P_{\bm \iota_{t}'|\bm h'_{k}, S, S'}\| P_{\mathring{\bm \iota}_{t}|\bm h_{k}, S, S'})
			&\leq -\dfrac{1}{2t}D_{KL}(P_{\bm \iota_{t}'|\bm h'_{k}, S, S'}\| P_{\mathring{\bm \iota}_{t}|\bm h_{k}, S, S'})+\dfrac{4A^2b^2}{n^2s^2_{k}}\\
			&\leq -\dfrac{1}{2t^{\alpha_k}}D_{KL}(P_{\bm \iota_{t}'|\bm h'_{k}, S, S'}\| P_{\mathring{\bm \iota}_{t}|\bm h_{k}, S, S'})+\dfrac{4A^2b^2}{n^2s^2_{k}}.
		\end{split}
	\end{equation*}	
	It follows that
	\begin{eqnarray*}
		D_{KL}(P_{\bm \iota_{t}'|\bm h'_{k}, S, S'}\| P_{\mathring{\bm \iota}_{t}|\bm h_{k}, S, S'})
		&\leq  \Big(e^{-\frac{1}{2(1-\alpha_k)}t^{1-\alpha_k}}\int_{0}^{t}e^{\frac{1}{2(1-\alpha_k)}u^{1-\alpha_k}}du\Big) \times \dfrac{4A^2b^2\eta_k}{n^2\sigma^2_{k}}.
	\end{eqnarray*} 
	Setting $t=\eta_k$, we get
	\[D_{KL}(P_{\bm \iota_{\eta_k}'|\bm h'_{k}, S, S'}\| P_{\mathring{\bm \iota}_{\eta_k}|\bm h'_{k}, S, S'})
	\leq \dfrac{ 4\delta_kA^2b^2\eta_k}{n^2\sigma^2_{k}},\]
	where $\delta_k=e^{-\frac{1}{2(1-\alpha_k)}\eta_k^{1-\alpha_k}}\int_{0}^{\eta_k }e^{\frac{1}{2(1-\alpha_k)}u^{1-\alpha_k}}du.$
	Combining this with (\ref{negho-02}), we have
	\begin{eqnarray*}
		\begin{split}
			D_{KL}(P_{\check{\bm{\theta}}'_k|S, S'}\|P_{\check{\bm{\theta}}_k|S, S'})&\leq D_{KL}(P_{\check{\bm{\theta}}'_{k-1}|S, S'}\|P_{\check{\bm{\theta}}_{k-1}|S, S'})+\dfrac{ 4\delta_kA^2b^2\eta_k}{n^2\sigma^2_{k}}\\
			&\leq \ldots\\
			&\leq \displaystyle \sum_{j=1}^k \dfrac{4\delta_j A^2b^2\eta_j}{n^2\sigma^2_{j}}.
		\end{split}
	\end{eqnarray*}	
	Combining it with the Pinsker's inequality $D_{TV}\leq \sqrt{D_{KL}/2}$, we have
	\[D_{TV}(P_{\check{\bm{\theta}}'_k|S, S'}\|P_{\check{\bm{\theta}}_k|S, S'})\leq  \sqrt{\displaystyle  \sum_{j=1}^k \dfrac{2\delta_j A^2b^2\eta_j}{n^2\sigma^2_{j}}}.\]
	Finally, we arrive at the following bound for the generalization error,
	\[|ge_k|\leq |\xi_{k}(\bm \theta_k,\bm \epsilon_k)|+2c_0\sqrt{\displaystyle \sum_{j=1}^k \dfrac{2\delta_j A^2b^2\eta_j}{n^2\sigma^2_{j}}}.\]

	\noindent\textit{Part 2.}
	Define $\mathring{\bm \iota}_{t}=S_{k}W_{t}$ and
	\[\bm \iota_{t}'=t[\tilde{\bm g}_{S',k}(\bm \theta'_{k-1})-\tilde{\bm g}_{S,k}(\bm \theta'_{k-1})]+S_{k}W'_t,\] where $S_{k}=\bm \Sigma_k/\eta_k.$ Then, we have
	\begin{equation*}
		\begin{split}
			&d\mathring{\bm \iota}_{t}=S_{k}dW_t~\text{and}~
			d\bm \iota_{t}'=\tilde{\bm g}_{S',k}(\bm \theta'_{k-1})-\tilde{\bm g}_{S,k}(\bm \theta'_{k-1})+S_{k}dW'_t.
		\end{split}
	\end{equation*}
	Let $ P_{\mathring{\bm \iota}_{t}|\bm h_{k}, S, S'}$ and $P_{\bm \iota_{t}'|\bm h'_{k}, S, S'}$ represent the density functions of $\mathring{\bm \iota}_{t}$ and $\bm \iota_{t}'$ given $\bm h'_{k}$, $S$ and $S'$, respectively. According to (\ref{2th1-eq1}), these two density functions adhere to the subsequent differential equation,
	\begin{equation}\label{2th1-eq11}
		\begin{split}
			&\dfrac{\partial}{\partial t} P_{\mathring{\bm \iota}_{t}|\bm h_{k}, S, S'}=\frac{S^2_{k}}{2}\bigtriangleup  P_{\mathring{\bm \iota}_{t}|\bm h_{k}, S, S'}\\
            \text{and}~~&\dfrac{\partial}{\partial t}P_{\bm \iota_{t}'|\bm h'_{k}, S, S'}=-\nabla\odot\big(F_{J,k}(\bm \iota_{t}')P_{\bm \iota_{t}'|\bm h'_{k}, S, S'}\big)+\frac{S^2_{k}}{2}\bigtriangleup P_{\bm \iota_{t}'|\bm h'_{k}, S, S'},
		\end{split}
	\end{equation}
	where  $F_{J,k}(\bm \iota_{t}')=\dfrac{b}{n}\mathbb E[\tilde{\bm g}_{S',k}(\bm \theta'_{k-1})-\tilde{\bm g}_{S,k}(\bm \theta'_{k-1})|\bm I_{k, j}=1, \bm \iota_{t}'=w,\bm h'_{k}, S, S'].$ By (\ref{2th1-def1}), $\|F_{J,k}(\bm \iota_{t}')\|\leq \dfrac{2Ab}{n}.$
	Following the proof of Theorem \ref{th-sub-2}, we obtain that 
	\begin{equation*}
		\begin{split}
			&~~~\dfrac{d}{dt}D_{KL}(P_{\bm \iota_{t}'|\bm h'_{k}, S, S'}\| P_{\mathring{\bm \iota}_{t}|\bm h_{k}, S, S'})\\
			&=\int \dfrac{d P_{\bm \iota_{t}'|\bm h'_{k}, S, S'}(u)}{dt}\left[\log\dfrac{P_{\bm \iota_{t}'|\bm h'_{k}, S, S'}(u)}{ P_{\mathring{\bm \iota}_{t}|\bm h_{k}, S, S'}(u)}-1\right]du-\int  \dfrac{d P_{\mathring{\bm \iota}_{t}|\bm h_{k}, S, S'}(u)}{d_{t}}\dfrac{P_{\bm \iota_{t}'|\bm h'_{k}, S, S'}(u)}{ P_{\mathring{\bm \iota}_{t}|\bm h_{k}, S, S'}(u)}du\\
			&\leq -\dfrac{\lambda_{k}}{2}\int  \bigg\|\nabla\log\dfrac{P_{\bm \iota_{t}'|\bm h'_{k}, S, S'}(u)}{ P_{\mathring{\bm \iota}_{t}|\bm h_{k}, S, S'}(u)}\bigg\|^2P_{\bm \iota_{t}'|\bm h'_{k}, S, S'}(u)du+\int F_{J,k}(\bm \iota_{t}')\odot\left(\nabla\log\dfrac{P_{\bm \iota_{t}'|\bm h'_{k}, S, S'}(u)}{ P_{\mathring{\bm \iota}_{t}|\bm h_{k}, S, S'}(u)}\right)P_{\bm \iota_{t}'|\bm h'_{k}, S, S'}(u)du\\
			&\leq -\dfrac{\lambda_{k}}{4}\int  \bigg\|\nabla\log\dfrac{P_{\bm \iota_{t}'|\bm h'_{k}, S, S'}(u)}{ P_{\mathring{\bm \iota}_{t}|\bm h_{k}, S, S'}(u)}\bigg\|^2P_{\bm \iota_{t}'|\bm h'_{k}, S, S'}(u)du+\dfrac{\eta_k}{\underline{\lambda}_{k}}\int\|F_{J,k}(u)\|^2P_{\bm \iota_{t}'|\bm h'_{k}, S, S'}(u)du\\
			&\leq -\dfrac{\lambda_{k}}{4}\int  \bigg\|\nabla\log\dfrac{P_{\bm \iota_{t}'|\bm h'_{k}, S, S'}(u)}{ P_{\mathring{\bm \iota}_{t}|\bm h_{k}, S, S'}(u)}\bigg\|^2P_{\bm \iota_{t}'|\bm h'_{k}, S, S'}(u)du+\dfrac{4A^2b^2\eta_k}{n^2\underline{\lambda}_{k}}.
		\end{split}
	\end{equation*}
	According to Theorem 1 in \citet{Markowich2000}, 
	the density function $ P_{\mathring{\bm \iota}_{t}|\bm h_{k}, S, S'}\sim N(0,t\Sigma_k)$ satisfies a logarithmic
	Sobolev inequality with a constant of $\eta_k/(\overline{\lambda}_{k}t),$ that is
	\begin{equation*}
		\int  \bigg\|\nabla\log\dfrac{P_{\bm \iota_{t}'|\bm h'_{k}, S, S'}(u)}{ P_{\mathring{\bm \iota}_{t}|\bm h_{k}, S, S'}(u)}\bigg\|^2P_{\bm \iota_{t}'|\bm h'_{k}, S, S'}(u)du\geq \dfrac{\eta_k}{t\overline{\lambda}_{k}}D_{KL}(P_{\bm \iota_{t}'|\bm h'_{k}, S, S'}\| P_{\mathring{\bm \iota}_{t}|\bm h_{k}, S, S'}).
	\end{equation*}
	Consequently, for any $t\in(0,1)$ and $\alpha_k\in(0,1)$, we have
	\begin{equation*}
		\begin{split}
			\dfrac{d}{dt}D_{KL}(P_{\bm \iota_{t}'|\bm h'_{k}, S, S'}\| P_{\mathring{\bm \iota}_{t}|\bm h_{k}, S, S'})
			&\leq -\dfrac{\underline{\lambda}_{k}}{2t\overline{\lambda}_{k}}D_{KL}(P_{\bm \iota_{t}'|\bm h'_{k}, S, S'}\| P_{\mathring{\bm \iota}_{t}|\bm h_{k}, S, S'})+\dfrac{4A^2b^2\eta_k}{n^2\underline{\lambda}_{k}}.\\
			&\leq -\dfrac{\underline{\lambda}_{k}}{2\overline{\lambda}_{k}}\dfrac{1}{t^{\alpha_k}}D_{KL}(P_{\bm \iota_{t}'|\bm h'_{k}, S, S'}\| P_{\mathring{\bm \iota}_{t}|\bm h_{k}, S, S'})+\dfrac{4A^2b^2\eta_k}{n^2\underline{\lambda}_{k}}.
		\end{split}
	\end{equation*}	
	It follows that
	\begin{eqnarray*}
		D_{KL}(P_{\bm \iota_{t}'|\bm h'_{k}, S, S'}\| P_{\mathring{\bm \iota}_{t}|\bm h_{k}, S, S'})
		&\leq \big(e^{-\frac{1}{2(1-\alpha_k)}t^{1-\alpha_k}}\int_{0}^{t}e^{\frac{1}{2(1-\alpha_k)}u^{1-\alpha_k}}du\big)\times \dfrac{4A^2b^2\eta_k}{n^2\underline{\lambda}_{k}}.
	\end{eqnarray*} 
	By setting $t=\eta_k$, 
	$$D_{KL}(P_{\bm \iota_{t}'|\bm h'_{k}, S, S'}\| P_{\mathring{\bm \iota}_{\eta_k}|\bm h'_{k}, S, S'})\leq \dfrac{4\zeta_kA^2b^2\eta_k}{n^2\underline{\lambda}_{k}}$$ 
	with $\zeta_k=e^{-\frac{\underline{\lambda}_{k}}{2\overline{\lambda}_{k}(1-\alpha_k)}\eta_k ^{1-\alpha_k}}\int_{0}^{\eta_k }e^{\frac{\underline{\lambda}_{k}}{2\overline{\lambda}_{k}(1-\alpha_k)}u^{1-\alpha_k}}du$. 
	Combining it with (\ref{negho-02}), we have
	\begin{equation*}
		\begin{split}
			D_{KL}(P_{\check{\bm{\theta}}'_k|S, S'}\|P_{\check{\bm{\theta}}_k|S, S'})\leq D_{KL}(P_{\check{\bm{\theta}}'_{k-1}|S, S'}\|P_{\check{\bm{\theta}}_{k-1}|S, S'})+ \dfrac{4\zeta_kA^2b^2\eta_k}{n^2\underline{\lambda}_{k}}.
		\end{split}
	\end{equation*}
	By the induction, we ascertain that 
	$$D_{KL}(\check{\pi}'_{k|S, S', R},\check{\pi}_{k|S, S', R}) \leq \displaystyle \sum_{s=1}^k\dfrac{4\zeta_sA^2b^2\eta_s}{n^2\underline{\lambda}_{s}}.$$ Furthermore, integrating it with Pinsker's inequality, we obtain
	\[D_{TV}(P_{\check{\bm{\theta}}'_k|S, S'}\|P_{\check{\bm{\theta}}_k|S, S'})\leq  \sqrt{\displaystyle \sum_{s=1}^k\dfrac{2\zeta_sA^2b^2\eta_s}{n^2\underline{\lambda}_{s}}}.\]	
	Then we derive the following bound		
	\[|ge_k|\leq |\xi_{k}(\bm \theta_k,\bm \epsilon_k)|+ 2c_0 \sqrt{\displaystyle \sum_{s=1}^k\dfrac{2\zeta_sA^2b^2\eta_s}{n^2\underline{\lambda}_{s}}}.\]
	Then the proof is finished.	
\end{proof}

\section*{Acknowledgments}
Li's research was supported by National Natural Science Foundation of China (NSFC) (Grant No. 12325110, 12288201) and CAS Project for Young Scientists in Basic Research (Grant No. YSBR-034).
Xiong's research was supported by NSFC (Grant No. 11801102).

\section*{Appendix}
The detailed parameters for MLP and AlexNet used in Section 4 are as follows.
The MLP network for MNIST consists of two fully connected layers. The first layer has 512 neurons (default width), taking a flattened input of size \(28 \times 28\) and applying a ReLU activation. 
The network used for CIFAR-10 is an optimized version of AlexNet tailored for CIFAR-10, using smaller convolutional kernels (\(3\times3\)) and increased filter sizes. It consists of five convolutional layers with filter sizes of 64, 128, 256, 512, and 256, respectively, followed by max-pooling, batch normalization, and ReLU activations. The fully connected layers include 1024 neurons each, with dropout regularization applied at a rate of 0.5. This architecture, designed to handle \(32\times32\times3\) CIFAR-10 images, improves feature extraction and training stability.

    During the simulations, we apply an early stopping strategy: if the test loss does not improve by more than \(10^{-4}\) compared to the previous best value for \(P\) consecutive epochs, training is halted. This helps prevent overfitting and ensures efficient convergence. Figure S1 shows the training and test accuracy, along with the loss, achieved with a width of 512 and \( P=3 \) on the MNIST dataset. The training process stabilizes after 30 epochs and concludes at 143 epochs, attaining a test accuracy of 98.03\%.     	
\begin{figure}[htbp]
	\centering
	\begin{subfigure}[b]{0.48\textwidth} 
		\centering
		\includegraphics[width=\textwidth]{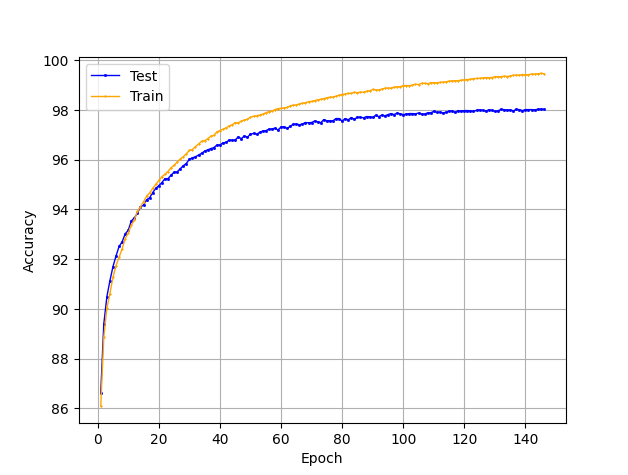}
		\caption{Accuracy}
		\label{fig:sub1}
	\end{subfigure}
	\hfill
	\begin{subfigure}[b]{0.48\textwidth} 
		\centering
		\includegraphics[width=\textwidth]{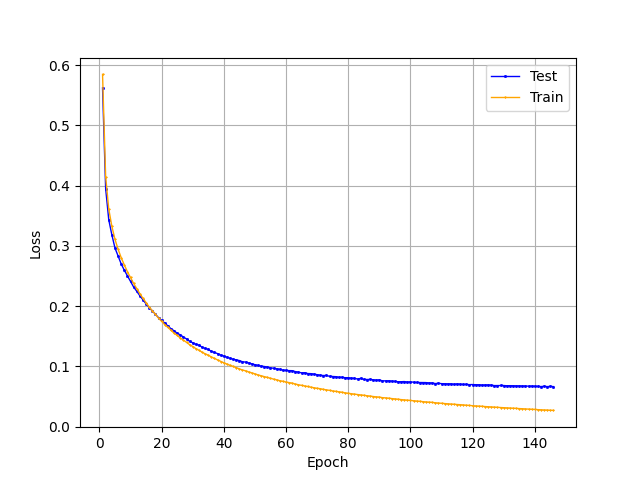}
		\caption{Loss}
		\label{fig:sub2}
	\end{subfigure}
	\caption*{Figure S1: Training and Test Accuracy and Loss for MLP on MNIST.}
	\label{fig5}
\end{figure}

Furthermore, we examine the impact of widths in MLP on the performance of our proposed method, {\it T2pm-SGD}, and illustrate the variation of the trajectory and flatness terms with width in Figure S2. The trajectory term, which captures the overall progress of the training, and the flatness term, which reflects the optimization landscape, exhibit contrasting behaviors as the width of the network increases. Specifically, the trajectory term decreases steadily with increasing width, indicating a smoother optimization path, while the flatness term shows a reduction in magnitude, suggesting that a wider network leads to a more stable and less sensitive loss landscape. These trends help us understand the relationship between network capacity and the stability of the optimization process.

\begin{figure}[htbp]
	\centering
	\begin{subfigure}[b]{0.48\textwidth} 
		\centering
		\includegraphics[width=\textwidth]{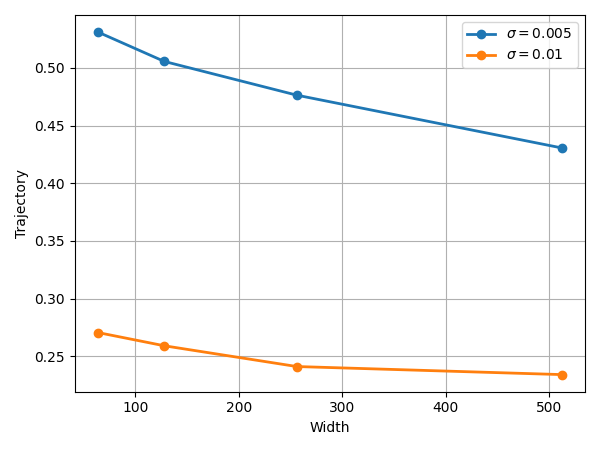}
		\caption{Trajectory}
		\label{fig:sub1}
	\end{subfigure}
	\hfill
	\begin{subfigure}[b]{0.48\textwidth} 
		\centering
		\includegraphics[width=\textwidth]{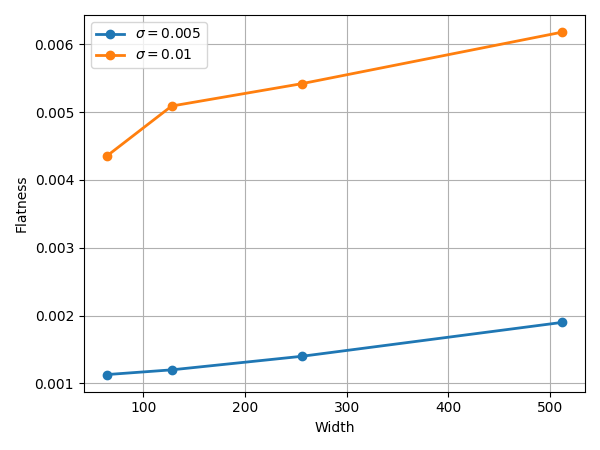}
		\caption{Flatness}
		\label{fig:sub2}
	\end{subfigure}
	\caption*{Figure S2: Variation of trajectory and flatness terms in the bound with width in SGD Training of MLP on MNIST.}
	\label{figS2}
\end{figure}

\vskip 0.2in

\bibliographystyle{model5-names}
\bibliography{SGDreferences}	
	
\end{document}